\newif\ifreport\reporttrue
\documentclass[conference,letterpaper]{IEEEtran}
\IEEEoverridecommandlockouts
\usepackage{setspace}
\usepackage{cite}
\usepackage{graphicx}
\usepackage{caption}
\usepackage{bbm}
\usepackage{subcaption}
\usepackage[cmex10]{amsmath}
\usepackage{hyperref}
\usepackage{mathtools}
\usepackage{amsthm}
\usepackage{amsfonts}
\usepackage{amssymb}
\usepackage{mathdots}
\usepackage{mathtools}
\usepackage{enumitem}
\usepackage{algorithm}
\usepackage{algorithmic}
\usepackage{bm,xstring}
\usepackage{fixltx2e}
\usepackage{tikz}
\usetikzlibrary{decorations.pathreplacing}
\usepackage{color}
\usepackage{tcolorbox}
\newtheorem{theorem}{Theorem}
\newtheorem{lemma}{Lemma}
\newtheorem{corollary}{Corollary}
\newtheorem{definition}{Definition}
 \newtheorem{assumption}{Assumption}

\newcommand{\ignore}[1]{{}}
\def\BibTeX{{\rm B\kern-.05em{\sc i\kern-.025em b}\kern-.08em
    T\kern-.1667em\lower.7ex\hbox{E}\kern-.125emX}}
\usepackage{color}

\def\red{\color{red}}

\begin{document}

\title{A Local Geometric Interpretation of Feature Extraction in Deep Feedforward Neural Networks\\
\thanks{This work was supported in part by the NSF grant CCF-1813078 and the ARO grant W911NF-21-1-0244.}
}

\author{\IEEEauthorblockN{Md Kamran Chowdhury Shisher, Tasmeen Zaman Ornee, and Yin Sun}
\IEEEauthorblockA{Deptartment of Electrical and Computer Engineering \\
Auburn University, Auburn, AL, 36849 \\
\{mzs0153, tzo0017, yzs0078\}@auburn.edu}
}

\maketitle
\begin{abstract}
In this paper, we present a local geometric analysis to interpret how deep feedforward neural networks extract low-dimensional features from high-dimensional data. Our study shows that, in a local geometric region, the optimal weight in one layer of the neural network and the optimal feature generated by the previous layer comprise a low-rank approximation of a matrix that is determined by the Bayes action of this layer. This result holds (i) for analyzing both the output layer and the hidden layers of the neural network, and (ii) for neuron activation functions with non-vanishing gradients.
We use two supervised learning problems to illustrate our results: neural network based maximum likelihood classification (i.e., softmax regression) and neural network based minimum mean square estimation. Experimental validation of these theoretical results will be conducted in our future work.
\end{abstract}
\ignore{Our local geometric analysis for feature extraction in Neural Networks is a unified framework that covers a wide range of inference problems with different loss functions and activation functions. We consider local geometric analysis around Bayes action. Moreover, our approach is applicable to problems consisting of any number of hidden layers. In addition, the features extracted by the intermediate layers have been solved. We propose an iterative algorithm to solve the deep learning feature extraction problem.} 
\ignore{We investigate a problem of finding low dimensional representations, also known as features, from high dimensional data for inference tasks. Specifically, we provide a local geometric analysis of features extracted in a deep feedforward neural network based supervised learning. The problem can be cast as a stochastic decision problem. The Bayes action of the decision problem plays a central role in the local analysis. Our study shows that the optimal feature generated by the last hidden layer and the optimal weight of the output layer corresponds to a low rank approximation of a matrix that characterizes the Bayes action. The analysis of the output layer can be applied to the hidden layers in an iterative manner: The features obtained in a later layer is a Bayes action of the optimization problem of the previous layer.  Our theoretical findings are applicable to twice continuously differentiable loss functions and to strictly increasing and continuously differentiable activation functions. Finally, we present two examples: (i) Neural network based maximum likelihood estimation and (ii) Neural network based minimum mean square estimation to illustrate our theoretical analysis.}

\section{introduction}
In recent years, neural network based supervised learning has been extensively admired due to its emerging applications in a wide range of inference problems, such as image classification, DNA sequencing, natural language processing, etc. The success of deep neural networks depends heavily on its capability of extracting good low-dimensional features from high-dimensional data. 
Due to the complexity of deep neural networks, theoretical interpretation of feature extraction in deep neural networks has been challenging, with some recent progress reported in, e.g., \cite{bartlett2019nearly, zhang2021understanding, karakida2019universal, mei2018mean, goldt2020modeling, jacot2018neural, lei2020geometric, geiger2021landscape, quinn2021information, huang2019information, tishby2015deep, yu2020learning, arora}.

In this paper, we analyze the training of deep feedforward neural networks for a class of empirical risk minimization (ERM) based supervised learning algorithms. A local geometric analysis is conducted for feature extraction in deep feedforward neural networks. Specifically, the technical contributions of this paper are summarized as follows:
\begin{itemize}
\item 
We first analyze the design of (i) the weights and biases in the output layer and (ii) the feature constructed by the last hidden layer. In a local geometric region, this design problem is converted to a low-rank matrix approximation problem, where the matrix is characterized by the Bayes action of the supervised learning problem. Optimal designs of the weights, biases, and feature are derived in the local geometric region (see Theorems \ref{theorem1}-\ref{theorem3}).

\item The above local geometric analysis can be readily applied to a hidden layer (see Corollaries \ref{corollary2}-\ref{corollary4}), by considering another supervised learning problem for the hidden layer. The local geometric analyses of different layers are  related to each other in an iterative manner: The optimal feature obtained from the analysis of one layer is the Bayes action needed for analyzing the previous layer. We use two supervised learning problems to illustrate our results.



\end{itemize}

\subsection{Related Work}
Due to the practical success of deep neural networks, there have been numerous efforts \cite{bartlett2019nearly, zhang2021understanding, karakida2019universal, mei2018mean, goldt2020modeling, jacot2018neural, lei2020geometric, geiger2021landscape, quinn2021information, huang2019information, tishby2015deep, yu2020learning, arora} to explain the feature extraction procedure of deep neural networks. Towards this end, researchers have used different approaches, for example, statistical learning theory approach \cite{bartlett2019nearly, zhang2021understanding}, information geometric approach \cite{karakida2019universal, mei2018mean, goldt2020modeling, jacot2018neural, lei2020geometric, geiger2021landscape, quinn2021information, huang2019information}, information theoretic approach \cite{huang2019information, tishby2015deep, yu2020learning}, etc. The information bottleneck formulation in \cite{tishby2015deep} suggested that the role of the deep neural network is to learn minimal sufficient statistics of the data for an inference task. The authors in \cite{yu2020learning} proposed that maximal coding rate reduction is a fundamental principle  in deep neural networks. In \cite{huang2019information}, the authors formulated the problem of feature extraction by using KL-divergence, and provided a local geometric analysis by considering a weak dependency between the data and the label. Motivated by \cite{huang2019information}, we also consider the weak dependency. Compared to  \cite{huang2019information}, our local geometric analysis can handle more general supervised learning problems and neuron activation functions, as explained in Section \ref{sec_3}.

\begin{figure}[t]
\centering
\includegraphics[width=0.49\textwidth]{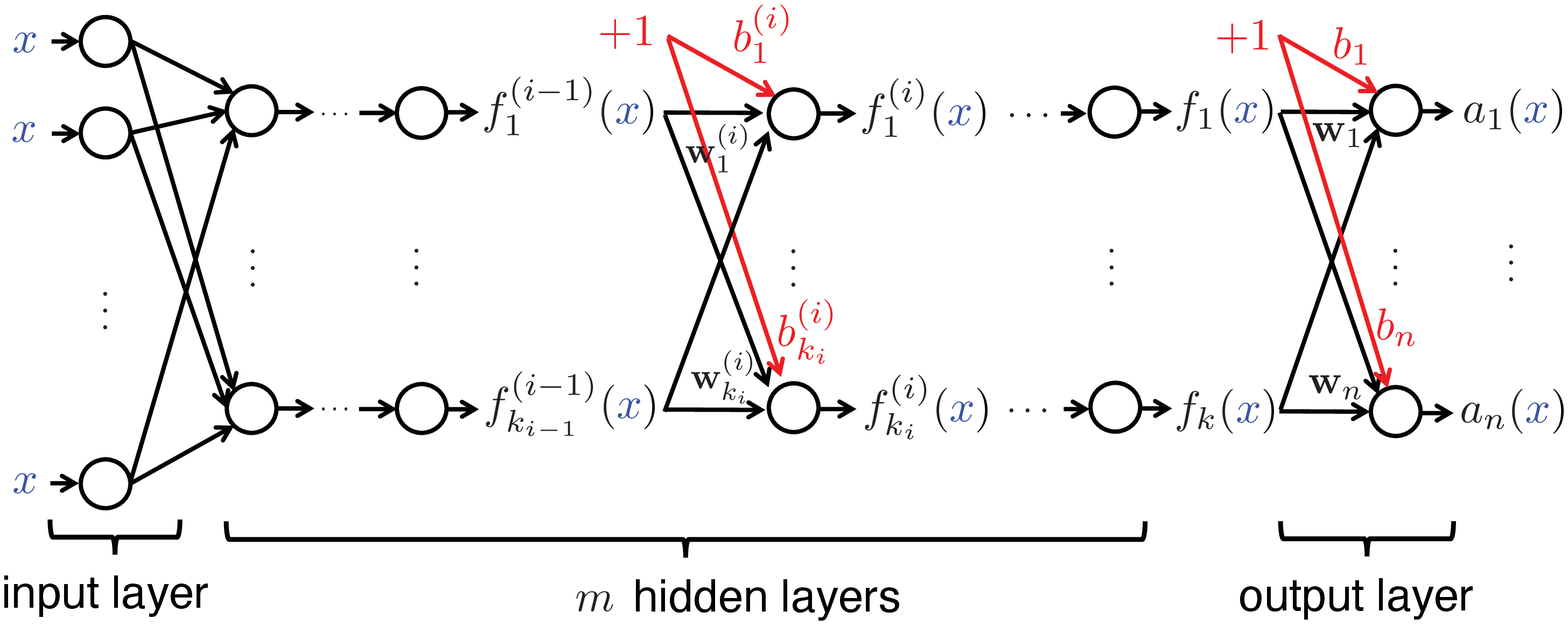}
\caption{\small A deep feedforward neural network. \label{fig:DNN}
}
\end{figure}

\section{Model and Problem}

\subsection{Deep Feedforward Neural Network Model} 
Consider the deep feedforward neural network illustrated in Figure \ref{fig:DNN}, which consists of one input layer, $m$ hidden layers, and one output layer. The input layer admits an input variable $x\in \mathcal X$ and feeds a vector 
\begin{align}\label{input}
\mathbf f^{(0)}(x) = [h^{(0)}_1(x), \ldots, h_{k_0}^{(0)}(x)]^{\operatorname{T}} \in \mathbb R^{k_0}
\end{align}
to the first hidden layer, where $k_0$ is the number of neurons in the input layer and $h^{(0)}_j: \mathcal X\mapsto \mathbb R$ is the activation function of the $j$-th neuron in the input layer. For all $i=1,\ldots,m$, the $i$-th hidden layer admits $\mathbf f^{(i-1)}(x)\in \mathbb R^{k_{i-1}}$ from the previous layer and constructs a vector $\mathbf f^{(i)}(x)\in \mathbb R^{k_{i}}$, usually called a \emph{feature}, given by
\begin{align}\label{featurei}
&\mathbf f^{(i)}(x)\nonumber\\
=&\!\left[h^{\!(i)}\!(\mathbf w_1^{(i)\!\operatorname{T}}\mathbf f^{(i-1)}\!(x)\!+\!b^{(i)}_1), \ldots, h^{\!(i)}\!(\mathbf w_{k_i}^{(i)\!\operatorname{T}}\mathbf f^{(i-1)}\!(x)\!+\!b^{(i)}_{k_i})\right]^{\operatorname{T}}\!\!, 
\end{align}
where $k_i$ is the number of neurons in the $i$-th hidden layer, $h^{(i)}: \mathbb R\mapsto \mathbb R$ is the activation function of each neuron in the $i$-th hidden layer, $\mathbf w_j^{(i)}\in \mathbb R^{k_{i-1}}$ and $b_j ^{(i)}\in \mathbb R$ are the weight vector and bias of the $j$-th neuron in the $i$-th hidden layer, respectively. Denote $\mathbf W^{(i)} =[\mathbf w^{(i)}_1, \ldots, \mathbf w^{(i)}_{k_{i}}]$ and $\mathbf b^{(i)}=$ $[b^{(i)}_1, \ldots, b^{(i)}_{k_i}]^{\operatorname{T}}$, then \eqref{featurei} can be expressed compactly as 
\begin{align}
\mathbf f^{(i)}(x)=\mathbf h^{(i)} \left(\mathbf W^{(i)\operatorname{T}} \mathbf f^{(i-1)}(x) + \mathbf b^{(i)}\right),
\end{align}
where $\mathbf h^{(i)}: \mathbb R^{k_i}\mapsto \mathbb R^{k_i}$ is a vector-valued function determined by \eqref{featurei}.
For notational simplicity, let us denote $k= k_m$ and $\mathbf f(x)=\mathbf f^{(m)}(x)$. The output layer admits $\mathbf f(x)\in \mathbb R^{k}$ from the last hidden layer and generates an output vector $\mathbf a(x) \in \mathcal H^n$, called an \emph{action}, which is determined by
\begin{align}\label{case1}
\mathbf a(x) =& {\mathbf h}(\mathbf {W}^{\operatorname{T}}\mathbf f(x)+ \mathbf b)\nonumber\\
=&  [h(\mathbf w_{1}^{\operatorname{T}} \mathbf f(x) + b_1), \cdots, h(\mathbf w_{n}^{\operatorname{T}} \mathbf f(x)  + b_n)]^{\operatorname{T}},
\end{align}
where $n$ is the number of neurons in the output layer, $h: \mathbb R \mapsto \mathcal H$ is the activation function of each neuron in the output layer, $\mathcal H$ is the image set of $h$ with $\mathcal H\subseteq \mathbb R$, $\mathbf w_j\in \mathbb R^{k}$ and $b_j \in \mathbb R$ are the weight vector and bias of the $j$-th neuron in the output layer, respectively, $\mathbf W =[\mathbf w_1, \ldots, \mathbf w_n]$ and $\mathbf b = [b_1, \ldots, b_n]^{\operatorname{T}}$. 

\ignore{Our system consists of a feedforward neural network. Consider a random variable $X \in \mathcal X$ that denotes the data, where $\mathcal X$ is a discrete finite set. As shown in Fig. \ref{fig:DNN}, for an input $x \in \mathcal{X}$, an $m$ hidden layered feedforward neural network generates $k$ dimensional feature $\mathbf f(x) = [f_1 (x), \ldots, f_k(x)]^{\operatorname{T}}$ and produces an action $\mathbf a(x) = [a_1(x), \ldots, a_n(x)]^{\operatorname{T}} \in \mathbb R^n$. In this sequel, we first explain the features generated by the hidden layers and demonstrate the action later in this section.

\ignore{The input to the first hidden layer is given by 
\begin{align}
\mathbf f^{(0)}(x) =& [\mathbbm1(x=1), \ldots, \mathbbm1(x=|\mathcal{X}|)] \in \mathbb{R}^{|\mathcal{X}|},
\end{align}
where $\mathbbm1(x=j)$ is an indicator function defined as
\begin{align}
\mathbbm1(x=j) =& \bigg\{\begin{array}{l l} 1,
& \text{ if }~ x = j,\\
0, & \text{ if }~ x \neq j.
\end{array}
\end{align}}
Let, $\mathbf f^{(i)}(x) \in \mathbb R^{k_i}$ denotes the feature produced by the first $i$ hidden layers for all $i=1, \ldots, m$. The $j$-th element of the vector $\mathbf f^{(i)}(x)$ is $$h^{(i)}(\mathbf w_j^{(i)\operatorname{T}}\mathbf f^{(i-1)}(x)+b^{(i)}_j),$$ where $\mathbf f^{(0)}(x)\in \mathbb R^{k_0}$ is the input of the neural network, $h^{(i)}:\mathcal X \mapsto \mathbb R$ is the activation function of the $i$-th hidden layer, $\mathbf w_j^{(i)} \in \mathbb R^{k_{i-1}}$ and $b_j^{(i)} \in \mathbb R$ are the weight vector and bias of the $j$-th node of the $i$-th hidden layer, respectively. Hence, $\mathbf f^{(i)}(x)$ can be expressed as
\begin{align}
\mathbf f^{(i)}(x)= \bm h^{(i)} \left(\mathbf W^{(i)\operatorname{T}} \mathbf f^{(i-1)}(x) + \mathbf b^{(i)}\right),
\end{align}
where $\mathbf W^{(i)}=[\mathbf w_1^{(i)}, \ldots, \mathbf w_{k_i}^{(i)}]\in \mathbb{R}^{k_{i-1} \times {k_i}}$ and $\mathbf b^{(i)} \in \mathbb{R}^{k_i}$ are weights and bias of the $i$-th hidden layer, respectively. 
The feature $\mathbf f^{(m)}(x)$ is the output of the last hidden layer. For notational simplicity, we will use $\mathbf f^{(m)}(x)=\mathbf f(x)\in \mathbb R^k$. 

An activation function $h: \mathbb R \mapsto \mathcal H$ is applied on $\mathbf w_i^{\operatorname{T}}\mathbf f(x)+b_i$, where $\mathbf w_i \in \mathbb R^k$ is the weight and $b_i \in \mathbb R$ is the bias of the output layer. Then, the feedforward neural network produces $\mathbf a(x) \in \mathcal A=\mathcal H^n$, where $\mathcal A$ is a finite set and $\mathbf a(x)$ can be represented as
\begin{align}\label{case1}
\mathbf a(x) =& {\mathbf h}(\mathbf {W}^{\operatorname{T}}\mathbf f(x)+ \mathbf b), \nonumber\\
=&  [h(\mathbf w_{1}^{\operatorname{T}} \mathbf f(x) + b_1), \cdots, h(\mathbf w_{n}^{\operatorname{T}} \mathbf f(x)  + b_n)]^{\operatorname{T}}.
\end{align}}

\subsection{Neural Network based Supervised Learning Problem}

The above deep feedforward neural network is used to solve a supervised learning problem. We focus on a class of popular supervised learning algorithms called \emph{empirical risk minimization (ERM)}. In ERM algorithms, the weights and biases of the neural network are trained to construct a vector-valued function $\mathbf a: \mathcal X \mapsto \mathcal A$ that outputs an action $\mathbf a(x)\in \mathcal A$ for each $x \in \mathcal X$, where $\mathcal A \subseteq \mathcal H^n \subseteq \mathbb R^n$. Consider two random variables $X\in\mathcal X$ and $Y\in\mathcal Y$, where $\mathcal X$ and $\mathcal Y$ are finite sets. The performance of an ERM algorithm is measured by a loss function $L: \mathcal Y \times \mathcal A \mapsto \mathbb R$, where $L(y,\mathbf a(x))$ is the incurred loss if action $\mathbf a(x)$ is generated by the neural network when $Y=y$. For example, in neural network based maximum likelihood classification, also known as \emph{softmax regression}, the loss function is 
\begin{align}\label{log-loss}
L_{\log}(y,\mathbf a) = -\log\left( \frac{a_y}{\sum_{y'\in \mathcal Y} a_{y'}}\right),
\end{align}
which is the negative log-likelihood of a distribution $Q_Y$ generated by the neural network, where $Q_Y(y)= {a_y}/{\sum_{y'\in \mathcal Y} a_{y'}}$, $a_y>0$ for all $y\in \mathcal Y$, and the dimension of $\mathbf a$ is  $n=|\mathcal Y|$. In neural network based minimum mean-square estimation, the loss function is one half of the mean-square error between $\mathbf y \in \mathbb R^n$ and an estimate $\hat{\mathbf y}=\mathbf a(x) \in \mathbb R^n$ constructed by the neural network, i.e.,
\begin{align}\label{mean-square-error}
L_2(\mathbf y,\hat{\mathbf y}) = \frac{1}{2} \|\mathbf y-\hat{\mathbf y} \|_2^2.
\end{align}
Let  $P_{X,Y}$ be the empirical joint distribution of $X$ and $Y$ in the training data,  $P_X$ and $P_Y$ be the associated marginal distributions, which satisfies  $P_X(x)>0$ for all $x\in\mathcal X$ and $P_Y(y)>0$ for all $y\in\mathcal Y$. The objective of ERM algorithms is to solve the following neural network training problem: 
\begin{align}\label{main_problem}
\min_{\substack{(\mathbf W, \mathbf b),\\~(\mathbf W^{(i)}, \mathbf b^{(i)}), i=1, \ldots, m}} \mathbb E_{X,Y \sim P_{X, Y}}[L\left(Y, \mathbf a(X)\right)],
\end{align}
where $\mathbf a(x)$ is subject to \eqref{input}-\eqref{case1}, because $\mathbf a(x)$ is the action generated by the neural network.

\ignore{Consider a random variable $Y \in \mathcal Y$ that represents the label, where $\mathcal Y$ is a discrete finite set. In supervised learning, a decision maker infers the label $Y$ by using the action $\mathbf a(x)$ provided by a neural network. The performance of the inference is measured by a loss function $L$. The $L(y, \mathbf{a})$ is the incurred loss if an action $\mathbf{a} \in \mathcal A$ is chosen when $Y=y$. For example, $L_2(y, \hat y)=\frac{1}{2}\|y-\hat y\|_2^2$ is a quadratic loss function and $L_{\text{log}}(y, \mathbf a)=-\text{log}(\frac{1}{\|\mathbf a\|_1} a_y)$ is a logarithmic loss function, where $\mathbf a \in \mathbb R^{|\mathcal Y|}$. Let $P_{X, Y}$ be the empirical joint distribution of the data $X$ and the label $Y$. Then, the objective of DNN based supervised learning is to solve the following optimization problem:
\begin{align}\label{main_problem}
\min_{\substack{(\mathbf W, \mathbf b),\\~(\mathbf W^{(i)}, \mathbf b^{(i)}),\\ i=1, \ldots, m}} \mathbb E_{X,Y \sim P_{X, Y}}[L\left(Y, \mathbf a(X)\right)],
\end{align}
where $\mathbf a(x)$ is action for given $x \in \mathcal X$ generated by the weight matrices $\mathbf W, \mathbf W^{(1)}, \ldots, \mathbf W^{(m)}$ and the bias vectors $\mathbf b, \mathbf b^{(1)}, \ldots, \mathbf b^{(m)}$ defined in \eqref{case1}.
}
\subsection{Problem Reformulation}
Denote $\Phi = \{f: \mathcal X\mapsto \mathcal A\}$ as the set of all  functions from $\mathcal X$ to $\mathcal A$. Any action function $\mathbf a(x)$ produced by the neural network, i.e., any function satisfying \eqref{input}-\eqref{case1}, belongs to $\Phi$, whereas some functions in $\Phi$ cannot be constructed by the neural network. By relaxing the set of feasible action functions  in \eqref{main_problem} as $\Phi$, we derive  the following lower bound of \eqref{main_problem}:
\begin{align}\label{lower_bound1}
&\min_{\mathbf a \in \Phi} \mathbb E_{X,Y \sim P_{X, Y}}[L(Y, \mathbf a(X))] \\
=&\sum_{x \in \mathcal X}P_X(x) \min_{\mathbf a(x) \in \mathcal A} \mathbb E_{Y \sim P_{Y|X=x}}[L(Y, \mathbf a(x))],\label{lower_bound}
\end{align}
where \eqref{lower_bound1} is decomposed into a sequence of separable optimization problems in \eqref{lower_bound}, each optimizing the action $\mathbf a(x)\in \mathcal A$ for a given $x\in\mathcal X$. 
Let $\mathcal A_{P_Y}\subseteq \mathcal A$ denote the set of optimal solutions to the following problem: 
\begin{align}\label{bayes}
\mathcal A_{P_Y} =\arg\min_{\mathbf a \in \mathcal A} \mathbb E_{Y \sim P_Y}[L(Y, \mathbf a)]
\end{align}
and use $\mathbf a_{P_Y}$ to denote an element of $\mathcal A_{P_Y}$, which is usually called a \emph{Bayes action}. Define the discrepancy 
\begin{align}\label{L-divergence}
D_L(\mathbf a_{P_Y}||\mathbf a)=\mathbb E_{Y \sim P_{Y}}[L(Y, \mathbf a)] -\mathbb E_{Y \sim P_{Y}}[L(Y, \mathbf a_{P_Y})].
\end{align} 
According to \eqref{bayes} and \eqref{L-divergence}, $D_L(\mathbf a_{P_Y}||\mathbf a)\geq 0$ for all $\mathbf a\in \mathcal A$, where equality is achieved if and only if $\mathbf a \in \mathcal A_{P_Y}$. {When $\mathbf a=\mathbf a_{Q_Y}$, $D_L(\mathbf a_{P_Y}||\mathbf a_{Q_Y})$ is a generalized divergence  between $P_Y$ and $Q_Y$ \cite{farnia2016minimax,grunwald2004game}.} 

By subtracting the lower bound \eqref{lower_bound} from \eqref{main_problem}, we obtain the following  problem that is equivalent to \eqref{main_problem}: 
\begin{align}\label{reformed_problem_old}
\min_{\substack{(\mathbf W, \mathbf b),\\~(\mathbf W^{(i)}, \mathbf b^{(i)}), i=1, \ldots, m}} \sum_{x \in \mathcal X} P_X(x) D_L(\mathbf a_{P_{Y|X=x}} || \mathbf a(x)),
\end{align}
where $\mathbf a_{P_{Y|X=x}}\in \mathcal A_{P_{Y|X=x}}$ is a Bayes action associated with the conditional empirical distribution $P_{Y|X=x}$ and $\mathbf a(x)$ is subject to \eqref{input}-\eqref{case1}. 

\ignore{Consider a set of functions $\Phi$ that contains all functions mapping from the input space $\mathcal X$ to $\mathcal A$. Then, a lower bound of \eqref{main_problem} is
\begin{align}\label{lower_bound}
&\min_{\mathbf a \in \Phi} \mathbb E_{X,Y \sim P_{X, Y}}[L(Y, \mathbf a(X))] \nonumber\\
=&\sum_{x \in \mathcal X}P_X(x) \min_{\mathbf a(x) \in \mathcal A}  E_{Y \sim P_{Y|X=x}}[L(Y, \mathbf a(x))].
\end{align}
In the right-hand side of \eqref{lower_bound}, the lower bound is decomposed into a sequence of separated optimization problems each optimizing action $\mathbf a(x)$ for a given $x \in \mathcal X$. The optimal action to \eqref{lower_bound} for a given $x \in \mathcal X$ is called Bayes action $\mathbf a_{P_{Y|X=x}}$ associated with the conditional distribution $P_{Y|X=x}$ \cite{farnia2016minimax, grunwald2004game, dawid2014theory}.

Let $\mathcal A_{P_Y}$ denotes the set of all Bayes actions $\mathbf a_{P_Y}$. If we need to choose an action $\mathbf a$ without the knowledge of $X$, we get Bayes action $\mathbf a_{P_Y}$:
\begin{align}\label{bayes}
\min_{\mathbf a \in \mathcal A} \mathbb E_{Y \sim P_Y}[L(Y, \mathbf a)]=\mathbb E_{Y \sim P_Y}[L(Y, \mathbf a_{P_Y})].
\end{align}

The excess risk $D_L(\mathbf a_{P_Y}||\mathbf a)$ of an action $\mathbf a \in \mathcal A$ is given by
\begin{align}\label{L-divergence}
D_L(\mathbf a_{P_Y}||\mathbf a)=\mathbb E_{Y \sim P_{Y}}[L(Y, \mathbf a)] -\mathbb E_{Y \sim P_{Y}}[L(Y, \mathbf a_{P_Y})].
\end{align} 
From the definitions \eqref{bayes} and \ref{L-divergence}, we get for all $a \in \mathcal A$,
\begin{align}\label{property}
D_L(\mathbf a_{P_Y}|| \mathbf a) \geq 0
\end{align}
with equality if and only if $\mathbf a = \mathbf a_{P_Y}$.

By subtracting the lower bound \eqref{lower_bound} from \eqref{main_problem}, the problem \eqref{main_problem} can be equivalently expressed as: 
\begin{align}\label{reformed_problem_old}
\min_{\substack{(\mathbf W, \mathbf b),\\~(\mathbf W^{(i)}, \mathbf b^{(i)}),\\ i=1, \ldots, m}} \sum_{x \in \mathcal X} P_X(x) D_L(\mathbf a_{P_{Y|X=x}} || \mathbf a(x)).
\end{align}
Instead of solving \eqref{lower_bound}, we seek to find weight matrices and bias vectors $(\mathbf W, \mathbf b)$ and $(\mathbf W^{(i)}, \mathbf b^{(i)})$, $i=1, \ldots, m$, which yield a good action $\mathbf a(x)$ and minimize the action divergence \eqref{reformed_problem_old}. 
}

\section{Main Results: Feature Extraction in \\ Deep Feedforward Neural Networks}\label{sec_3}
\subsection{Local Geometric Analysis of the Output Layer}\label{feature_extraction}
We  consider the following reformulation of \eqref{reformed_problem_old} that focuses on the training of the output layer: 
\begin{align}\label{reformed_problem}
\min_{\substack{\mathbf{W} \in  \mathbb R^{k \times n},\\ \mathbf b \in \mathbb R^n,\mathbf f \in \Lambda}} ~\sum_{x \in \mathcal X} P_X(x) D_L(\mathbf a_{P_{Y|X=x}} || \mathbf h(\mathbf{W}^{\operatorname{T}}\mathbf f(x)+\mathbf b)),
\end{align}
where $\Lambda$ is the set of feature functions created by the input and hidden layers of the neural network. 

Recall that $\mathcal H^n$ is the image set of the vector-valued activation function $\mathbf h(\mathbf b)$ of the output layer. Because $\mathcal A_{P_Y}\subseteq \mathcal A\subseteq \mathcal H^n$, for any Bayes action $\mathbf a_{P_Y}\in\mathcal A_{P_Y}$  that solves \eqref{bayes}, there exists a bias $\tilde{\mathbf b}=[\tilde b_1,\ldots,\tilde b_n]^{\operatorname{T}}\in \mathbb R^n$ such that 
\begin{align}\label{BiasVector}
\mathbf h(\tilde{\mathbf b})=[ h(\tilde b_1),\ldots,h(\tilde b_n)]^{\operatorname{T}}=\mathbf a_{P_Y}.
\end{align}
The following assumption is needed in our study.
\ignore{
\begin{assumption}\label{assumption1}
The activation function $h$ is strictly increasing and continuously differentiable. 
\end{assumption}

Moreover, according to the above arguments in \eqref{eq_continuouslydifferentiable}-\eqref{eq_alpha_i_bound}, if Assumption \ref{assumption1} is replaced by the following weaker Assumption \ref{assumption6}, then Lemma \ref{lemma1} can be also proven: 
}
\begin{assumption}\label{assumption1}
For each $i=1,\ldots, n$, there exist $\delta>0$ and $K>0$ such that for all $z\in (\tilde b_i-\delta, \tilde b_i+\delta)$, the  activation function $h$ satisfies
\begin{align}\label{eq_assumption1}
\left| h(z) - h(\tilde b_i)\right|\geq K \left|z-\tilde b_i\right|.
\end{align}
\end{assumption}

\begin{lemma}\label{assumption6to1}
If $h$ is strictly increasing and continuously differentiable, then $h$ satisfies Assumption \ref{assumption1}.
\end{lemma}
\ifreport
\begin{proof}
See Appendix \ref{passumption6to1}.
\end{proof}
\else
Due to space limitation, all the proofs are relegated to our
technical report \cite{TechnicalReport}.\!\!
\fi
It is easy to see that the leaky ReLU activation function \cite[pp. 187-188]{goodfellow2016deep} satisfies Assumption \ref{assumption1}. In addition, the hyperbolic tangent function and the sigmoid function \cite[p. 189]{goodfellow2016deep}
also satisfy Assumption \ref{assumption1}, because they are strictly increasing and continuously differentiable. 

Let $\mathcal P^{\mathcal Y}$ be the set of all probability distributions on $\mathcal Y$ and $\text{relint}(\mathcal P^{\mathcal Y})$ be the relative interior of the set $\mathcal P^{\mathcal Y}$.

\begin{assumption}\label{assumption5}
If two distributions $P_Y, Q_Y \in \text{\normalfont{relint}}(\mathcal P^{\mathcal Y})$ are close to each other such that
\begin{align}
\sum_{y \in \mathcal Y} (P_Y(y)-Q_Y(y))^2\leq \gamma^2,
\end{align}
then for any  $\mathbf a_{P_Y} \in \mathcal A_{P_Y}$, there exists an $\mathbf a_{Q_Y} \in \mathcal A_{Q_Y}$ such that 
\begin{align}
\|\mathbf a_{P_Y}-\mathbf a_{Q_Y}\|_2=O(\gamma).
\end{align}
\end{assumption}
Assumption \ref{assumption5} characterizes the differentiability of the Bayes action $\mathbf a_{P_Y}$ with respect to $P_Y$. 
The loss functions in \eqref{log-loss} and \eqref{mean-square-error} satisfy Assumption \ref{assumption5}, as explained later in Section \ref{examples}.

Because of the universal function approximation properties of deep feedforward neural networks \cite{cybenko1989approximation, hornik1989multilayer, goodfellow2016deep}, we make the following assumption. 

 \begin{assumption}\label{assumption3}
For given  $\epsilon >0$ and  $\mathbf a_{P_{Y|X=x}}\in\mathcal A_{P_{Y|X=x}}$ with $x\in \mathcal X$, there exists an optimal solution $(\mathbf f, \mathbf W, \mathbf b)$ to \eqref{reformed_problem} such that for all $x\in \mathcal X$
\begin{align}\label{eq_assumption3}
\| \mathbf a_{P_{Y|X=x}} - \mathbf h({\mathbf{W}}^{\operatorname{T}}{\mathbf f}(x)+{\mathbf b})\|_2^2 \leq \epsilon^2.
\end{align}
\end{assumption}
By Assumption 
\ref{assumption3}, the neural network can closely approximate the vector-valued function $x\mapsto \mathbf a_{P_{Y|X=x}}$.

 \begin{definition}For a given $\epsilon >0$, two random variables $X$ and $Y$ are called {$\epsilon$-dependent}, if the $\chi^2$-mutual information $I_{\chi^2}(X;Y)$ is no more than $\epsilon^2$, given by
\begin{align}\label{weak-dependent}
I_{\chi^2}(X;Y)=D_{\chi^2}(P_{X,Y} || P_X \otimes P_Y)\leq \epsilon^2,
\end{align}
where 
\begin{align}
D_{\chi^2}(P_X||Q_X)=\int_{\mathcal X} \frac{(P(x)-Q(x))^2}{Q^2(x)} d Q(x)
\end{align}
is Neyman's $\chi^2$-divergence \cite{polyanskiy2014lecture}. 
\end{definition}

Motivated by the seminal work \cite{huang2019information} and \cite{huang2019universal}, we consider the following assumption. 

\begin{assumption}\label{assumption4}
For a given $\epsilon >0$, $X$ and $Y$ are $\epsilon$-dependent. 
\end{assumption}

By using the above assumptions, we can find a local geometric region \eqref{constraint} that is useful for our analysis.

\begin{lemma}\label{lemma1}
For a sufficiently small $\epsilon > 0$, if Assumptions \ref{assumption1}-\ref{assumption4} hold, then there exists an optimal solution $(\mathbf f, \mathbf{W}, \mathbf b)$ to \eqref{reformed_problem}  such that for all $x\in\mathcal X$ and $i=1, \ldots, n$
\begin{align}\label{constraint}
{\mathbf w_i}^{\operatorname{T}}\mathbf f(x)+ b_i-\tilde b_i=O(\epsilon).
\end{align}
\end{lemma}
\ifreport
\begin{proof}
See Appendix \ref{plemma1}.
\end{proof}
\else
\fi
\ignore{[Proof Sketch]
By Assumptions \ref{assumption5} and \ref{assumption4}, there exists $a_{P_Y}\in \mathcal A_{P_Y}$ and $a_{P_{Y|X=x}}\in \mathcal A_{P_Y}$ such that $a_{P_Y}$ and $a_{P_{Y|X=x}}$ are within an $O(\epsilon)$ distance from each other for all $x\in \mathcal X$. By Assumption \eqref{assumption3}, there exists an optimal solution to \eqref{reformed_problem} such that $\mathbf h({\mathbf{W}}^{\operatorname{T}}{\mathbf f}(x)+{\mathbf b})$ and $\mathbf a_{P_{Y|X=x}}$ are within an $O(\epsilon)$ distance from each other for all $x\in \mathcal X$. Hence, $\mathbf h({\mathbf{W}}^{\operatorname{T}}{\mathbf f}(x)+{\mathbf b})$ and $\mathbf a_{P_Y}$ are within an $O(\epsilon)$ distance. Because $\mathcal A \subseteq H^n$, there exists a bias $\tilde{\mathbf b}$ such that $\mathbf h(\tilde{\mathbf{b}}) = \mathbf a_{P_Y}$. Finally, by Assumption \ref{assumption1}, we can find a lower bound of the gradient $h'(\tilde{b}_i)$ in the neighborhood of $\mathbf b= \tilde{\mathbf b}$, which implies \eqref{constraint}. \ifreport For details, see Appendix \ref{plemma1}.} 
%

For any feature $\mathbf f(x)\in\mathbb R^{k}$, define a matrix  $\mathbf{\Xi}_{\mathbf f} \in \mathbb R^{k \times |\mathcal X|}$ as 
\begin{align}\label{matrixf}
\mathbf{\Xi}_{\mathbf f}=[\bm \xi_{\mathbf f}(1), \ldots, \bm \xi_{\mathbf f}(|\mathcal X|)], 
\end{align} 
where 
\begin{align}\label{vectorf}
\bm \xi_{\mathbf f}(x)=&\sqrt{P_X(x)} \left(\mathbf f(x)-\bm \mu_{\mathbf f}\right), \\ \label{meanf}
\bm \mu_{\mathbf f}=&\sum_{x \in \mathcal X} P_X(x) \mathbf f(x).
\end{align}  
In addition, define the following matrix $\mathbf B \in \mathbb R^{n \times |\mathcal X|}$ based on  the Bayes actions $\mathbf a_{P_{Y|X=x}}$ for $x\in \mathcal X$:
\begin{align}\label{matrixB}
\mathbf{B}=[\bm{\beta}_Y(1), \ldots, \bm \beta_Y(|\mathcal X|)],
\end{align}
where 
\begin{align}\label{vectorB}
\bm \beta_Y(x)=&\sqrt{P_X(x)} \left(\mathbf a_{P_{Y|X=x}}-\bm \mu_{\mathbf a}\right),\\ \label{mu_a}
\bm \mu_{\mathbf a}=&\sum_{x \in \mathcal X} P_X(x) \mathbf a_{P_{Y|X=x}}.
\end{align} 

\begin{assumption}\label{assumption2}
The function $\mathbf a \mapsto L(y, \mathbf a)$ is twice continuously differentiable. 
\end{assumption}
The Hessian matrix $\mathbf{M}_L$ of the function $\mathbf a \mapsto \mathbb E_{Y \sim P_Y}[L(Y, \mathbf a)]$ at the point $\mathbf a= \mathbf a_{P_Y}$ is 
\begin{align}\label{matrixM}
\mathbf{M}_L=\frac{\partial^2 \mathbb E_{Y \sim P_Y}[L(Y, \mathbf a)]}{\partial \mathbf a \partial \mathbf a^{\operatorname{T}}}\bigg|_{\mathbf a=\mathbf a_{P_Y}}.
\end{align}
Because $\mathbf a_{P_Y}$ is an optimal solution to \eqref{bayes}, $\mathbf{M}_L$ is positive semi-definite. Hence, it has a Cholesky decomposition  $\mathbf{M}_L=\mathbf{R}_L^{\operatorname{T}}\mathbf{R}_L.$ The Jacobian matrix of  $\mathbf h(\mathbf b)$ at the point $\mathbf b=\tilde{\mathbf b}$ is 
\begin{align}\label{matrixJ}
\mathbf J= \frac{\partial \mathbf h(\mathbf b)}{\partial \mathbf b^{\operatorname{T}}}\bigg|_{\mathbf b=\tilde{\mathbf b}}.
\end{align} 
\begin{lemma}\label{lemma2}
If Assumptions \ref{assumption5}, \ref{assumption4}, and \ref{assumption2} are satisfied, then in the local analysis regime \eqref{constraint}, the objective function in \eqref{reformed_problem} can be
expressed as
\begin{align}\label{approximation1}
& \sum_{x \in \mathcal X} P_X(x) D_L(\mathbf a_{P_{Y|X=x}}||\mathbf h(\mathbf{W}^{\operatorname{T}}\mathbf f(x)+\mathbf b))\nonumber\\
=&\frac{1}{2} \|\mathbf{\tilde B}- \mathbf \Xi_{\mathbf W} \mathbf{\Xi}_{\mathbf f}\|_{F}^2+\frac{1}{2}\eta(\mathbf d, \mathbf f)+o(\epsilon^2),
\end{align}
where $\mathbf{\tilde B}=\mathbf{R}_L\mathbf{B}$, 
\begin{align}\label{weight_change}
\mathbf \Xi_{\mathbf W}&=\mathbf{R}_L\mathbf{J}\mathbf{W}^{\operatorname{T}},\\\label{bias_change}
\mathbf d&=\mathbf b-\tilde{\mathbf b}, \\ \label{eta}
\eta(\mathbf d, \mathbf f)&=(\mathbf a_{P_Y}-\bm \mu_{\mathbf a}+\mathbf J \mathbf d+\mathbf J \mathbf W^{\operatorname{T}}\bm \mu_{\mathbf f})^{\operatorname{T}}\mathbf{M}_L\nonumber\\
&~~~\times (\mathbf a_{P_Y}-\bm \mu_{\mathbf a}+\mathbf J \mathbf d+\mathbf J \mathbf W^{\operatorname{T}}\bm \mu_{\mathbf f}).
\end{align}
\end{lemma}
\ifreport
\begin{proof}
See Appendix \ref{plemma2}.
\end{proof}
\else
{}
\fi 
In the local analysis regime, the training of $(\mathbf f, \mathbf W, \mathbf b)$ in \eqref{reformed_problem} can be expressed as the following optimization problem of $(\mathbf \Xi_{\mathbf W}, \mathbf{\Xi}_{\mathbf f}, \bm \mu_{\mathbf f}, \mathbf d)$: 
\begin{align}\label{approximate_problem}
\min_{\substack{\mathbf \Xi_{\mathbf W}, \mathbf{\Xi}_{\mathbf f}, \bm \mu_{\mathbf f}, \mathbf d}} \frac{1}{2} \|\mathbf{\tilde B}- \mathbf \Xi_{\mathbf W} \mathbf{\Xi}_{\mathbf f}\|_{F}^2+\frac{1}{2}\eta(\mathbf d, \mathbf f).
\end{align}
When $(\mathbf{\Xi}_{\mathbf f}, \bm \mu_{\mathbf f})$ are fixed, the optimal $({\mathbf \Xi}_{\mathbf W}^*, {\mathbf d}^*)$ are determined by 
\begin{theorem}\label{theorem1}
For fixed $\mathbf{\Xi}_{\mathbf f}$ and $\bm \mu_{\mathbf f}$, the optimal ${\mathbf \Xi}_{\mathbf W}^*$ to minimize \eqref{approximate_problem} is given by
\begin{align}\label{optimal_weight_1}
{\mathbf \Xi}_{\mathbf W}^*=\tilde{\mathbf B} \mathbf {\Xi}_{\mathbf f}^{\operatorname{T}}(\mathbf \Xi_{\mathbf f}\mathbf \Xi_{\mathbf f}^{\operatorname{T}})^{\operatorname{-1}},
\end{align}
and the optimal bias ${\mathbf d}^*$ is expressed as
\begin{align}\label{optimal_bias}
{\mathbf d}^*=-{\mathbf W}^{\operatorname{T}}\bm \mu_{\mathbf f}+\mathbf{J}^{\operatorname{-1}}(\bm \mu_{\mathbf a}-\mathbf a_{P_Y}).\end{align}
\end{theorem}
\ifreport
\begin{proof}
See Appendix \ref{ptheorem1}.
\end{proof}
\else
{}
\fi
By Theorem \ref{theorem1}, the rows of ${\mathbf \Xi}_{\mathbf W}^*$ are obtained by projecting the rows of $\tilde{\mathbf B}$ on the subspace spanned by the rows of $\mathbf \Xi_{\mathbf f}$. The optimal bias $\mathbf d^*$ cancels out the effects of the mean feature $\bm \mu_{\mathbf f}$ and the mean difference $\bm \mu_{\mathbf a} -\mathbf a_{P_Y}$ between the Bayes actions $\mathbf a_{P_{Y|X=x}}$ and $\mathbf a_{P_Y}$. The optimal weight $\mathbf W^*$ and bias $\mathbf b^*$ can be derived by using  \eqref{weight_change}-\eqref{bias_change} and \eqref{optimal_weight_1}-\eqref{optimal_bias}. 

When $(\mathbf \Xi_{\mathbf W}, \mathbf d)$ are fixed and the hidden layers have sufficient expression power, the optimal $({\mathbf \Xi}_{\mathbf f}^*, \bm \mu_{\mathbf f}^*)$ are given by

\begin{theorem}\label{theorem2}
For fixed $\mathbf \Xi_{\mathbf W}$ and $\mathbf d$, the optimal ${\mathbf{\Xi}}_{\mathbf f}^*$ to minimize \eqref{approximate_problem} is given by
\begin{align}\label{optimal_feature}
{\mathbf \Xi}_{\mathbf f}^*=(\mathbf \Xi_{\mathbf W}^{\operatorname{T}}\mathbf \Xi_{\mathbf W})^{\operatorname{-1}}\mathbf \Xi_{\mathbf W}^{\operatorname{T}}\tilde{\mathbf B},
\end{align}
and the optimal mean ${\bm \mu}_{\mathbf f}^*$ is given by
\begin{align}\label{optimal_mean}
{\bm \mu}_{\mathbf f}^*=-(\mathbf \Xi_{\mathbf W}^{\operatorname{T}}\mathbf \Xi_{\mathbf W})^{\operatorname{-1}}\mathbf \Xi_{\mathbf W}^{\operatorname{T}}(\mathbf a_{P_Y}-\bm \mu_{\mathbf a}+\mathbf J\mathbf d).
\end{align}
\end{theorem}
\ifreport
\begin{proof}
See Appendix \ref{ptheorem2}.
\end{proof}
\else
{}
\fi

By Theorem \ref{theorem2}, the columns of ${\mathbf \Xi}_{\mathbf f}^*$ are obtained by projecting the columns of $\tilde{\mathbf B}$ on the subspace spanned by the columns of ${\mathbf \Xi}_{\mathbf W}$. The optimal feature $\bm \mu_{\mathbf f}^*$ cancels out the effects of $\mathbf d$ and $\bm \mu_{\mathbf a} -\mathbf a_{P_Y}$. The optimal feature $\mathbf f^*(x)$ can be derived by using \eqref{matrixf}-\eqref{meanf} and \eqref{optimal_feature}-\eqref{optimal_mean}. 

\ignore{According to Theorem \ref{theorem2}, {\red [Add your discussion of Theorem 2 here, which should be similar to the discussions for Theorem 1]}.} 

The singular value decomposition of $\tilde{\mathbf B}$ can be written as 
\begin{align}
\tilde{\mathbf B}=\mathbf U \mathbf \Sigma \mathbf V^{\operatorname{T}}, 
\end{align}
where $\mathbf \Sigma=\text{Diag}(\sigma_1, \ldots, \sigma_K)$ is a diagonal matrix with $K=\min(n, |\mathcal X|)$ singular values $\sigma_1\geq \sigma_2 \geq \ldots \geq \sigma_K=0$,  $\mathbf U$ and $\mathbf V$ are composed by the $K$ leading left and right singular vectors of $\tilde{\mathbf B}$, respectively. Denote
\begin{align}\label{px}
\sqrt{\mathbf p_X}=[\sqrt{P_X(1)}, \ldots, \sqrt{P_X(|\mathcal X|)}]^{\operatorname{T}}.
\end{align}
Because $\tilde{\mathbf B} \sqrt{\mathbf p_X}=0$ and $\|\sqrt{\mathbf p_X}\|_2=1$, 
$\sqrt{\mathbf p_X}$ is the right singular vector of $\tilde{\mathbf B}$ for  the singular value $\sigma_K=0$.
When $({\mathbf{\Xi}}_{\mathbf f}, \bm \mu_{\mathbf f}, {\mathbf \Xi}_{\mathbf W}, \mathbf d)$ are all designable, the optimal solutions are characterized in the following theorem. 

\begin{theorem}\label{theorem3}
If $k\leq \min(n, |\mathcal X|)$, then any $({\mathbf{\Xi}}_{\mathbf f}^*, {\mathbf \Xi}_{\mathbf W}^*)$ satisfying \eqref{optimal_pair1} jointly minimizes \eqref{approximate_problem}:
\begin{align}\label{optimal_pair1}
{\mathbf \Xi}_{\mathbf W}^*{\mathbf{\Xi}}_{\mathbf f}^*=\mathbf{U}_{k}\mathbf{\Sigma}_{k}\mathbf V_k^{\operatorname{T}},
\end{align}
where $\mathbf{\Sigma}_{k}=\textbf{\normalfont{Diag}}(\sigma_1, \ldots \sigma_k)$, $\mathbf U_k=[\mathbf u_1, \ldots, \mathbf u_k]$, and $\mathbf V_k=[\mathbf v_1, \ldots, \mathbf v_k]$.
Moreover, any bias ${\mathbf d}^*$ and mean ${\bm \mu}_{\mathbf f}^*$ satisfying \eqref{optimal_bias_mean} jointly minimizes \eqref{approximate_problem}\normalfont{:}
\begin{align}\label{optimal_bias_mean}
\mathbf{J}\left({\mathbf d}^*+{\mathbf W}^{\operatorname{T}}{\bm \mu}_{\mathbf f}^*\right)=\bm \mu_{\mathbf a}-\mathbf a_{P_Y}.\end{align}
\end{theorem}
\ifreport
\begin{proof}
See Appendix \ref{ptheorem3}.
\end{proof}
\else
{}
\fi
According to Theorem \ref{theorem3}, the optimal $({\mathbf{\Xi}}_{\mathbf f}^*, {\mathbf \Xi}_{\mathbf W}^*)$ are given by the  low-rank approximation of $\tilde{\mathbf B}$, which can be derived by using the power iteration algorithm  \cite{bulirsch2002introduction}, or equivalently, by executing \eqref{optimal_weight_1} and \eqref{optimal_feature} iteratively. 
The optimal $(\mathbf d^*,\bm \mu_{\mathbf f}^*)$ cancel out the effect of  $\bm \mu_{\mathbf a} - \mathbf a_{P_Y}$. 

The optimal  ${\mathbf \Xi}_{\mathbf f}^*$ in Theorems \ref{theorem2}-\ref{theorem3} can be achieved only when the hidden layers have sufficient expression power. 
Nonetheless,  ${\mathbf{\Xi}}_{\mathbf f}^*$ plays an important role in the analysis of the hidden layers, as explained in the next subsection.



\subsection{Local Geometric Analysis of Hidden Layers}
Next, we provide a local geometric analysis for each hidden layer. To that end, let us consider the training of the $i$-th hidden layer for fixed  weights and biases in the  subsequent layers. Define a loss function $L^{(i)}$ for the $i$-th hidden layer 
%
\begin{align}
L^{(i)}(y, \mathbf a^{(i)})=&L\left(y, \mathbf g \circ \mathbf g^{(m)}\circ \ldots \circ \mathbf g^{(i+1)}(\mathbf a^{(i)})\right),
\end{align}
where for $k=i,\ldots, m-1$
\begin{align} 
\mathbf g^{(k+1)}(\mathbf a^{(k)})&=\mathbf h^{(k+1)}(\mathbf W^{(k+1)\operatorname{T}}\mathbf a^{(k)}+\mathbf b^{(k+1)}),\\
\mathbf g(\mathbf a^{(m)})&=\mathbf h(\mathbf W^{\operatorname{T}}\mathbf a^{(m)}+\mathbf b).
\end{align} 
Given $(\mathbf W^{(k)}, \mathbf b^{(k)})$ for $k=i+1,\ldots,m$ and $(\mathbf W, \mathbf b)$, the training problem of the $i$-th hidden layer is formulated as 
\begin{align} \label{new_obj_func}
\min_{\substack{\mathbf{W}^{(i)},\\\mathbf b^{(i)},\\\mathbf f^{(i-1)}\in \Lambda^{i-1}}}\!\!\!\!\!\!\sum_{x \in \mathcal X}\!\! P_X(x) D_{\!L^{\!(i)}}\!(\mathbf a^{(i)}_{P_{Y\!|\!X=x}}\! || \mathbf h^{(i)}\!(\mathbf{W}^{(i)\!\operatorname{T}}\mathbf f^{(i-1)}\!(x)\!+\!\mathbf b^{\!(i)})),
\end{align}
where $\Lambda^{i-1}$ is the set of all feature functions that can be created by the first $(i-1)$ hidden layers. 
We adopt several assumptions for the $i$-th hidden layer that are similar to Assumptions \ref{assumption1}-\ref{assumption2}. Let $\mathbf a^{(i)}_{P_Y}$ denote the Bayes action associated to the loss function $L^{(i)}$ and distribution $P_Y$. 
According to Lemma \ref{lemma1}, there exists a bias $\tilde{\mathbf b}^{(i)}$ and a tuple $(\mathbf f^{(i-1)}, \mathbf{W}^{(i)}, \mathbf b^{(i)})$ such that (i) ~$\mathbf h^{(i)}(\tilde{\mathbf b}^{(i)})=\mathbf a^{(i)}_{P_Y}$ is a Bayes action associated to the loss function $L^{(i)}$ and distribution $P_Y$, (ii) $(\mathbf f^{(i-1)}, \mathbf{W}^{(i)}, \mathbf b^{(i)})$ is an optimal solution to \eqref{new_obj_func}, and (iii) for all $x\in\mathcal X$ and $j=1, \ldots, k_i$
\begin{align}\label{constraint2}
{\mathbf w_j}^{(i)\operatorname{T}}\mathbf f^{(i-1)}(x)+ b_j^{(i)}-\tilde b_j^{(i)}=O(\epsilon).
\end{align}

\ifreport
Define 
\begin{align}
& \mathbf \Xi_{\mathbf f^{(i)}}=[\bm \xi_{\mathbf f^{(i)}}(1), \ldots, \bm \xi_{\mathbf f^{(i)}}(|\mathcal X|)], \label{matrixf2} \\
& \mathbf{B}^{(i)}=[\bm{\beta}_Y^{(i)}(1), \ldots, \bm \beta_Y^{(i)}(|\mathcal X|)], \label{matrixB2}
\end{align}
where in \eqref{matrixf2}, 
\begin{align}\label{vectorf2}
\bm \xi_{\mathbf f^{(i)}}(x)=&\sqrt{P_X(x)} \left(\mathbf f^{(i)}(x)-\bm \mu_{\mathbf f^{(i)}}\right), \\ \label{meanf2}
\bm \mu_{\mathbf f^{(i)}}=&\sum_{x \in \mathcal X} P_X(x) \mathbf f^{(i)}(x),
\end{align}  
and in \eqref{matrixB2},
\begin{align}\label{vectorB}
\bm \beta_Y^{(i)}(x)=&\sqrt{P_X(x)} \left(\mathbf a^{(i)}_{P_{Y|X=x}}-\bm \mu_{\mathbf a^{(i)}}\right),\\
\bm \mu_{\mathbf a^{(i)}}=&\sum_{x \in \mathcal X} P_X(x) \mathbf a^{(i)}_{P_{Y|X=x}}.
\end{align} 
Similar to \eqref{matrixM} and \eqref{matrixJ}, let us define the following two matrices for the $i$-th hidden layer
\begin{align}\label{matrixM2}
\mathbf{M}_{L^{(i)}}=\frac{\partial^2 \mathbb E_{Y \sim P_Y}[L^{(i)}(Y, \mathbf a^{(i)})]}{\partial \mathbf a \partial \mathbf a^{\operatorname{T}}}\bigg|_{\mathbf a=\mathbf a^{(i)}_{P_Y}},
\end{align}
where the matrix $\mathbf{M}_{L^{(i)}}$ has a Cholesky decomposition $\mathbf{M}_{L^{(i)}}=\mathbf{R}_{L^{(i)}}^{\operatorname{T}}\mathbf{R}_{L^{(i)}}$ and 
\begin{align}
\mathbf J^{(i)}=\frac{\partial \mathbf h^{(i)}(\mathbf b^{(i)})}{\partial \mathbf b^{(i)\operatorname{T}}}\bigg|_{\mathbf b^{(i)}=\tilde{\mathbf b}^{(i)}}.
\end{align}
\else
Similar to \eqref{matrixf}-\eqref{matrixJ}, we define ${\mathbf \Xi}_{\mathbf f^{(i)}}$, $\mathbf B^{(i)}$, ${\bm \mu}_{\mathbf f^{(i)}}$, ${\bm \mu}_{\mathbf a^{(i)}}$, ${\mathbf M}_{L^{(i)}}$, ${\mathbf R}_{L^{(i)}}$, and $\mathbf J^{(i)}$ for the $i$-th hidden layer. Due to space limitations, these definitions are relegated to our technical report \cite{TechnicalReport}.
\fi
The following result is an immediate corollary of Lemma \ref{lemma2}.

\begin{corollary}\label{corollary1}
In the local analysis regime \eqref{constraint2}, the objective function in  \eqref{new_obj_func} can be expressed as
\begin{align}
&\sum_{x \in \mathcal X} \!P_X(x) D_{\!L^{(i)}}\!(\mathbf a^{(i)}_{P_{Y\!|\!X=x}}\! || \mathbf h^{(i)}(\mathbf{W}^{(i)\!\operatorname{T}}\mathbf f^{(i-1)}(x)\!+\!\mathbf b^{\!(i)}))\nonumber\\
=&\frac{1}{2} \|\tilde{\mathbf B}^{(i)}- {\mathbf \Xi}_{\mathbf W^{(i)}} {\mathbf{\Xi}}_{\mathbf f^{(i-1)}}\|_{F}^2+\frac{1}{2}\eta({\mathbf d}^{(i)},{\mathbf f}^{(i-1)})+o(\epsilon^2),
\end{align}
where $\tilde{\mathbf B}^{(i)}=\mathbf R_{L^{(i)}}{\mathbf B}^{(i)}$, ${\mathbf \Xi}_{\mathbf W^{(i)}}=\mathbf R_{L^{(i)}}\mathbf J^{(i)}\mathbf W^{(i)\operatorname{T}}$, $\mathbf d^{(i)}=\mathbf b^{(i)}-\tilde{\mathbf b}^{(i)}$, and
\begin{align}
&\eta(\mathbf d^{(i)}, \mathbf f^{(i)})\nonumber\\
=&(\mathbf a^{(i)}_{P_Y}-\bm \mu_{\mathbf a^{(i)}}+\mathbf J^{(i)} \mathbf d^{(i)}+\mathbf J^{(i)} \mathbf W^{(i)\operatorname{T}}\bm \mu_{\mathbf f^{(i-1)}})^{\operatorname{T}}
\mathbf{M}_L^{(i)}\nonumber\\
&\times (\mathbf a^{(i)}_{P_Y}-\bm \mu_{\mathbf a^{(i)}}+\mathbf J^{(i)} \mathbf d^{(i)}+\mathbf J^{(i)} \mathbf W^{(i)\operatorname{T}}\bm \mu_{\mathbf f^{(i-1)}}).
\end{align}
\end{corollary}

In the local analysis regime, the training of $(\mathbf \Xi_{\mathbf W^{(i)}}, \mathbf{\Xi}_{\mathbf f^{(i-1)}}, \mathbf d^{(i)}, \bm \mu_{{\mathbf f}^{(i-1)}})$ in \eqref{new_obj_func} can be expressed as the following optimization problem: 
\begin{align}\label{approximate_problem1}
\min_{\substack{\mathbf \Xi_{\mathbf W^{(i)}}, \mathbf{\Xi}_{\mathbf f^{(i-1)}}\\ \mathbf d^{(i)}, \bm \mu_{{\mathbf f}^{(i-1)}}}} \frac{1}{2} \|\tilde{\mathbf B}^{(i)}- \mathbf \Xi_{\mathbf W^{(i)}} \mathbf{\Xi}_{\mathbf f^{(i-1)}}\|_{F}^2+\frac{1}{2}\eta(\mathbf d^{(i)}, \mathbf f^{(i-1)}).
\end{align}
Similar to Theorems \ref{theorem1}-\ref{theorem3}, we can get 

\begin{corollary}\label{corollary2}
For fixed $\mathbf{\Xi}_{\mathbf f}^{(i-1)}$ and $\bm \mu_{\mathbf f^{(i-1)}}$, the optimal ${\mathbf \Xi}_{\mathbf W^{(i)}}^*$ to minimize \eqref{approximate_problem1} is given by
\begin{align}\label{optimal_weight_2}
{\mathbf \Xi}_{\mathbf W^{(i)}}^*=\tilde{\mathbf B}^{(i)} \mathbf {\Xi}_{\mathbf f}^{(i-1)\operatorname{T}}(\mathbf \Xi_{\mathbf f}^{(i-1)}\mathbf \Xi_{\mathbf f}^{(i-1)\operatorname{T}})^{\operatorname{-1}},
\end{align}
and the optimal bias ${\mathbf d}^{(i)*}$ is expressed as
\begin{align}\label{optimal_bias2}
{\mathbf d^{(i)*}}=-\bar{\mathbf W}^{(i)\operatorname{T}}\bm \mu_{\mathbf f^{(i-1)}}+(\mathbf{J}^{(i)})^{\operatorname{-1}}(\bm \mu_{\mathbf a^{(i)}}-\mathbf a^{(i)}_{P_Y}).\end{align}
\end{corollary}

\begin{corollary}\label{corollary3}
For fixed $\mathbf \Xi_{\mathbf W^{(i)}}$ and $\mathbf d^{(i)}$, the optimal ${\mathbf{\Xi}}_{\mathbf f^{(i-1)}}^*$ to minimize \eqref{approximate_problem1} is given by
\begin{align}\label{optimal_feature2}
{\mathbf \Xi}_{\mathbf f^{(i-1)}}^*=(\mathbf \Xi_{\mathbf W^{(i)}}^{\operatorname{T}}\mathbf \Xi_{{\mathbf W}^{(i)}})^{\operatorname{-1}}\mathbf \Xi_{\mathbf W^{(i)}}^{\operatorname{T}}\tilde{\mathbf B}^{(i)},
\end{align}
and the optimal mean ${\bm \mu}_{\mathbf f}^*$ is given by
\begin{align}\label{optimal_mean2}
{\bm \mu}_{\mathbf f^{(i-1)}}^*=-(\mathbf \Xi_{\mathbf W^{(i)}}^{\operatorname{T}}\mathbf \Xi_{{\mathbf W}^{(i)}})^{\operatorname{-1}}\mathbf \Xi_{\mathbf W^{(i)}}^{\operatorname{T}}   (\mathbf a^{(i)}_{P_Y}-\bm \mu_{\mathbf a}^{(i)}+\mathbf J^{(i)}\mathbf d^{(i)}).
\end{align}
\end{corollary}

\begin{corollary}\label{corollary4}
If $k_{i-1}\leq \min(k_i, |\mathcal X|)$, then any $({\mathbf{\Xi}}_{\mathbf f^{(i)}}^*, {\mathbf \Xi}_{\mathbf W^{(i)}}^*)$ satisfying \eqref{optimal_pair2} jointly minimizes \eqref{approximate_problem1}:
\begin{align}\label{optimal_pair2}
{\mathbf \Xi}_{\mathbf W^{(i)}}^*{\mathbf{\Xi}}_{\mathbf f^{(i)}}^*=\mathbf{U}^{(i)}_{k_{i-1}}\mathbf{\Sigma}^{(i)}_{k_{i-1}}\mathbf V_{k_{i-1}}^{(i)\operatorname{T}},
\end{align}
where $\mathbf{\Sigma}^{(i)}_{k_{i-1}}=\textbf{\normalfont{Diag}}(\sigma^{(i)}_1, \ldots \sigma^{(i)}_{k_{i-1}})$ is a diagonal matrix associated with $k_{i-1}$ leading singular values of $\tilde{\mathbf B}^{(i)}$, $\mathbf{U}^{(i)}_{k_{i-1}}$ and $\mathbf V_{k_{i-1}}^{(i)}$ are composed by the corresponding left and right singular vectors of $\tilde{\mathbf B}^{(i)}$, respectively.
Moreover, any bias ${\mathbf d}^{(i)*}$ and mean ${\bm \mu}_{\mathbf f^{(i)}}^*$ satisfying \eqref{optimal_bias_mean2} jointly minimizes \eqref{approximate_problem1}\normalfont{:}
\begin{align}\label{optimal_bias_mean2}
\mathbf{J}^{(i)}\left({\mathbf d}^{(i)*}+{\mathbf W^{(i)}}^{\operatorname{T}}{\bm \mu}_{\mathbf f^{(i)}}^*\right)=\bm \mu_{\mathbf a^{(i)}}-\mathbf a^{(i)}_{P_Y}.\end{align}
\end{corollary}

Compared to the local geometric analysis for softmax regression in \cite{huang2019information}, Theorems \ref{theorem1}-\ref{theorem3} and Corollaries \ref{corollary2}-\ref{corollary4} could handle more general loss functions and activation functions. In addition, our results can be applied to multi-layer neural networks in the following iterative manner: For fixed $(\mathbf W, \mathbf b)$ in the output layer, the Bayes action $\mathbf a^{(m)}_{P_{Y|X=x}}$ needed for analyzing the $m$-th hidden layer is  the optimal feature $\mathbf f^*(x)$ provided by Theorem \ref{theorem2}. 
Similar results hold for the $i$-th hidden layer. For fixed weights and biases in subsequent layers, the Bayes action $\mathbf a^{(i-1)}_{P_{Y|X=x}}$ needed for analyzing the $(i-1)$-th hidden layer is  the optimal feature $\mathbf f^{(i-1)*}(x)$ in Corollary \ref{corollary3}. 
Hence, the optimal features obtained in Theorem \ref{theorem2} and Corollary \ref{corollary3} are useful for the local geometric analysis of earlier layers. 

\ignore{In addition to the assumptions of the previous section, we consider the following assumptions.

\ignore{\begin{assumption}\label{assumption6}
The activation function $\psi$ is strictly increasing and continuously differentiable. 
\end{assumption}}

\ignore{\begin{assumption}\label{assumption7}
For a sufficiently small $\epsilon >0$ and all $\mathbf a_{P_{Y|X=x}}$, there exists an optimal solution $(\bar{\mathbf f}^{(1)}, \bar{\mathbf{W}}^{(1)}, \bar{\mathbf b}^{(1)})$ to \eqref{reformed_problem} that closely approximates $\mathbf a_{P_{Y|X=x}}$ such that 
\begin{align}
\!\!\!\!\| \mathbf a_{P_{Y|X=x}} - \mathbf h({\mathbf{W}}^{\operatorname{T}}\bm \psi({\bar{\mathbf W}^{(1){\operatorname{T}}}}\bar{\mathbf f}^{(1)}(x)+\bar{\mathbf b}^{(1)})+\mathbf b)\|_2^2\! \leq \epsilon^2.
\end{align}
\end{assumption}}

There exists vectors $\hat{\mathbf b}^{(1)}\in \mathbb R^{k}$ and $\hat {\mathbf b} \in \mathbb R^n$ that satisfies
\begin{align}\label{bias2}
\mathbf h(\mathbf W^{\operatorname{T}}\bm \psi(\hat{\mathbf b}^{(1)})+\hat {\mathbf b})=\mathbf a_{P_Y}.
\end{align} 

For a given $(\mathbf W, \mathbf b)$, we focus on the following local region, where $(\mathbf f^{(1)}, \mathbf W^{(1)}, \mathbf b^{(1)})$ satisfies
\begin{align}\label{constraint1}
\|{\mathbf{W}}^{\operatorname{T}}\bm \psi({\mathbf W}^{(1){\operatorname{T}}}{\mathbf f}^{(1)}(x)+{\mathbf b}^{(1)})+\mathbf b-\hat{\mathbf b}\|_2=O(\epsilon),
\end{align}
\begin{align}\label{constraint2}
\|{{\mathbf W}^{(1){\operatorname{T}}}}{\mathbf f}^{(1)}(x)+{\mathbf b}^{(1)}-\hat{\mathbf b}^{(1)}\|_2=O(\epsilon).
\end{align}

For differentiable activation functions $h$ and $\psi$, define Jacobian matrices  
\begin{align}
\hat{\mathbf J}=\frac{\partial \mathbf h(\mathbf b)}{\partial \mathbf b^{\operatorname{T}}}\bigg|_{\mathbf b=\hat{\mathbf b}},~~
\mathbf J_1=\frac{\partial \bm \psi(\mathbf b^{(1)})}{\partial {\mathbf b^{(1)}}^{\operatorname{T}}}\bigg|_{\mathbf b^{(1)}=\hat{\mathbf b}^{(1)}}.
\end{align}
Also, define a bias vector $\mathbf c \in \mathbb R^n$
\begin{align}
\mathbf c=\mathbf W^{\operatorname{T}}\mathbf J_1(\mathbf b^{(1)}-\hat{\mathbf b}^{(1)})+\mathbf b-\hat{\mathbf b}  
\end{align}
and a matrix
\begin{align}
\mathbf \Xi_{\mathbf f^{(1)}}=[\bm \xi_{\mathbf f^{(1)}}(1), \ldots, \bm \xi_{\mathbf f^{(1)}}(|\mathcal X|)],
\end{align}
where $\bm \xi_{\mathbf f^{(1)}}(x)=\sqrt{P_X(x)}(\mathbf f^{(1)}(x)-\sum_{x \in \mathcal X}P_X(x) \mathbf f^{(1)}(x))$.

\begin{lemma}\label{local_approximation_2}
Given \eqref{constraint1} and \eqref{constraint2}, the minimum excess risk \eqref{excess_risk_1} can be expressed as 
\begin{align}\label{approximation1}
&\!\!\!\!\!\!\!\!\!\!\sum_{x \in \mathcal X} P_X(x) D_L(\mathbf a_{P_{Y|X=x}} || \mathbf h(\mathbf{W}^{\operatorname{T}}\bm \psi({\mathbf W}^{(1)\operatorname{T}} {\mathbf f}^{(1)}(x)+{\mathbf b}^{(1)})+\mathbf b))\nonumber\\
=&\frac{1}{2} \|\mathbf{\tilde B}- {\mathbf \Xi}_{\mathbf W^{(1)}} {\mathbf{\Xi}}_{\mathbf f^{(1)}}\|_{F}^2+\frac{1}{2}\kappa({\mathbf c}, {\mathbf f}^{(1)})+o(\epsilon^2),
\end{align}
where ${\mathbf \Xi}_{\mathbf W}=\mathbf{R}_L\tilde{\mathbf{J}}\mathbf{W}^{\operatorname{T}}\mathbf J_1 {{\mathbf W}}^{(1)\operatorname{T}},$
\begin{align}
\kappa({\mathbf c}, {\mathbf f}^{(1)})=&(\mathbf R_L(\mathbf a_{P_Y}-\bm \mu_{\mathbf a}+\tilde{\mathbf J} {\mathbf c})+{\mathbf \Xi}_{\mathbf W^{(1)}}{\bm \mu}_{\mathbf f^{(1)}})^{\operatorname{T}} \nonumber\\
&\times(\mathbf R_L(\mathbf a_{P_Y}-\bm \mu_{\mathbf a}+\tilde{\mathbf J} {\mathbf c})+{\mathbf \Xi}_{\mathbf W^{(1)}}{\bm \mu}_{\mathbf f^{(1)}}),
\end{align}
\begin{align}
{\bm \mu}_{\mathbf f^{(1)}}&=\sum_{x \in \mathcal X} P_X(x) {\mathbf f}^{(1)}(x).
\end{align}
\end{lemma}
\begin{proof}
See Appendix \ref{plocal_approximation_2}.
\end{proof}
We solve the following problem 
\begin{align}\label{approximate_problem1}
\min_{\substack{\mathbf \Xi_{\mathbf W}, \mathbf \Xi_{\mathbf f^{(1)}}, \\ \mathbf c, \bm \mu_{\mathbf f^{(1)}}}} \frac{1}{2} \|\mathbf{\tilde B}- \mathbf \Xi_{\mathbf W^{(1)}} \mathbf{\Xi}_{\mathbf f^{(1)}}\|_{F}^2+\frac{1}{2}\kappa(\mathbf c, \mathbf f^{(1)}),
\end{align}
which is similar to \eqref{approximate_problem}.

\begin{theorem}\label{theorem4}
For fixed $\mathbf \Xi_{\mathbf f^{(1)}}$ and $\bm \mu_{\mathbf f^{(1)}}$, the optimal $\bar{\mathbf \Xi}_{\mathbf W^{(1)}}$ to minimize \eqref{approximate_problem1} is given by
\begin{align}\label{optimal_weight_2}
\bar{\mathbf \Xi}_{\mathbf W^{(1)}}=\tilde{\mathbf B} \mathbf {\Xi}_{\mathbf f^{(1)}}^{\operatorname{T}}(\mathbf \Xi_{\mathbf f^{(1)}}\mathbf \Xi_{\mathbf f^{(1)}}^{\operatorname{T}})^{\operatorname{-1}},
\end{align}
and the optimal bias $\bar{\mathbf c}$ is expressed as
\begin{align}\label{optimal_bias2}
\bar{\mathbf c}=-\bar{\mathbf W}^{\operatorname{T}}\mathbf J_{1}{\bar{\mathbf W}}^{(1)\operatorname{T}}\bm \mu_{\mathbf f^{(1)}}+\tilde{\mathbf{J}}^{\operatorname{-1}}(\bm \mu_{\mathbf a}-\mathbf a_{P_Y}).
\end{align}
\end{theorem}
\begin{proof}
See Appendix \eqref{ptheorem4}
\end{proof}

\begin{theorem}\label{theorem5}
For fixed $\mathbf W$, $\mathbf W^{(1)}$, and $\mathbf c$, the optimal $\bar{\mathbf{\Xi}}_{\mathbf f^{(1)}}$ to minimize \eqref{approximate_problem1} is given by
\begin{align}\label{optimal_feature2}
\bar{\mathbf \Xi}_{\mathbf f^{(1)}}=(\mathbf \Xi_{\mathbf W}^{(1)\operatorname{T}}\mathbf \Xi_{\mathbf W^{(1)}})^{\operatorname{-1}}\mathbf \Xi_{\mathbf W^{(1)}}^{\operatorname{T}}\tilde{\mathbf B}.
\end{align}
and the optimal mean $\bar{\bm \mu}_{\mathbf f^{(1)}}$ is given by
\begin{align}\label{optimal_mean2}
\mathbf J_{1}{\bar{\mathbf W}}^{(1)\operatorname{T}}\bar{\bm \mu}_{\mathbf f^{(1)}}=-(\mathbf \Xi_{\mathbf W}^{\operatorname{T}}\mathbf \Xi_{\mathbf W})^{\operatorname{-1}}\mathbf \Xi_{\mathbf W}^{\operatorname{T}}(\mathbf a_{P_Y}-\bm \mu_{\mathbf a}+\mathbf J\mathbf c).
\end{align}
\end{theorem}
\begin{proof}
See Appendix \ref{ptheorem5}.
\end{proof}

\begin{theorem}\label{theorem6}
If $k_1\leq K$, then an optimal $(\bar{\mathbf{\Xi}}_{\mathbf f^{(1)}}, \bar{\mathbf \Xi}_{\mathbf W^{(1)}})$ to minimize  \eqref{approximate_problem1} is given by
\begin{align}\label{optimal_pair2}
\bar{\mathbf{\Xi}}_{\mathbf f^{(1)}}&=\mathbf V_{k_1}^{\operatorname{T}} \\
\bar{\mathbf \Xi}_{\mathbf W^{(1)}}&=\mathbf{U}_{k_1}\mathbf{\Sigma}_{k_1},
\end{align} 
and the optimal $\bar{\mathbf c}$ and $\bar{\bm \mu}_{\mathbf f^{(1)}}$ should satisfy
\begin{align}
\bar{\mathbf c}=-\bar{\mathbf W}^{\operatorname{T}}\mathbf J_{1}{\bar{\mathbf W}}^{(1)\operatorname{T}}\bm \bar{\mu}_{\mathbf f^{(1)}}+\tilde{\mathbf{J}}^{\operatorname{-1}}(\bm \mu_{\mathbf a}-\mathbf a_{P_Y}).
 \end{align}
\end{theorem}
\begin{proof}
See Appendix \ref{ptheorem6}.
\end{proof}}


\subsection{Two Examples}\label{examples}
\subsubsection{Neural Network based Maximum Likelihood Classification (Softmax Regression)}
The Bayes actions $\mathbf a_{P_Y}$ associated to the loss function  \eqref{log-loss} are non-unique. The set of all Bayes actions is  $\mathcal A_{P_Y}=\{\alpha P_Y: \alpha >0\}$,
which satisfies Assumption \ref{assumption5}. By choosing one Bayes action $\mathbf a_{P_Y}=P_Y$, one can derive the matrices $\mathbf M_L$ and $\mathbf B$ used in Theorems \ref{theorem1}-\ref{theorem3}: The $(y, y')$-th element of $\mathbf M_L$ is
\begin{align}\label{example111}
(\mathbf M_L)_{y, y'}=\frac{\delta(y, y')}{P_Y(y)}-1,
\end{align}
where $\delta(y, y')=1$, if $y=y'$; and $\delta(y, y')=0$, if $y\neq y'$.
The $(y, x)$-th element of $\mathbf B$ is 
\begin{align}\label{example112}
(\mathbf B)_{y, x}=\sqrt{P_X(x)}(P_{Y|X=x}(y|x)-P_Y(y)).
\end{align}
To make our analysis applicable to the softmax activation function \cite[Eq. (6.29)]{goodfellow2016deep}, we have used a loss function \eqref{log-loss} that is different from the log-loss function in \cite{huang2019information,farnia2016minimax}. As a result, our local geometric analysis with \eqref{example111} and \eqref{example112} is different from the results in \cite{huang2019information}. 

\subsubsection{Neural Network based Minimum Mean-square Estimation}
Consider the minimum mean-square estimation of a random vector $\mathbf Y=[ Y_1,\ldots, Y_n]^{\operatorname{T}}$. The Bayes action associated to the loss function  \eqref{mean-square-error} is $\mathbf a_{P_{\mathbf Y}}=\mathbb E[\mathbf Y]$, which satisfies Assumption \ref{assumption5} because $\mathbb E[\mathbf Y]$ is a linear function of $P_{\mathbf Y}$. One can show that $\mathbf M_{L}=\mathbf I$ is an identity matrix and the $(j,x)$-th element of $\mathbf B$ is 
\begin{align}
(\mathbf B)_{j,x}=\sqrt{P_X(x)}(\mathbb E[ Y_j|X=x]-\mathbb E[ Y_j]).
\end{align}

\section{Conclusion}
In this paper, we have analyzed feature extraction in deep feedforward neural networks in a local region. We will conduct experiments to verify these results in our future work.

\bibliographystyle{IEEEtran}
\bibliography{ref2}
\appendices
\section{Proof of Lemma \ref{assumption6to1}}\label{passumption6to1}

Because $h$ is strictly increasing, $h'(\tilde b_i)>0$. Since $h$ is continuously differentiable, there exists a $\delta > 0$ such that for all  $z \in (\tilde b_i - \delta, \tilde b_i +\delta)$
\begin{align}\label{eq_continuouslydifferentiable1}
|h'(z)-h'(\tilde b_i)| \leq \frac{h'(\tilde b_i)}{2}. 
\end{align} 
It follows from  \eqref{eq_continuouslydifferentiable1} that for all  $z \in (\tilde b_i - \delta, \tilde b_i +\delta)$
\begin{align}\label{lemma1_ed}
h'(z) \geq \frac{h'(\tilde b_i)}{2}. 
\end{align}
From \eqref{lemma1_ed}, we can get that for all  $z \in (\tilde b_i - \delta, \tilde b_i +\delta)$
\begin{align}\label{lemma1_ecompare1}
\frac{|h(z) - h(\tilde b_i)|}{|z-\tilde b_i|} &\geq \frac{h'(\tilde b_i)}{2}.
\end{align}
Let $K = \frac{h'(\tilde b_i)}{2}$, then \eqref{lemma1_ecompare1} implies \eqref{eq_assumption1}, which completes the proof.


\section{Proof of Lemma \ref{lemma1}}\label{plemma1}

%
%
%

\begin{lemma}\label{norm_B}
If Assumptions \ref{assumption5} and \ref{assumption4} hold, then for any action $\mathbf a_{P_Y} \in \mathcal A_{P_Y}$ and  any $x \in \mathcal X$, there exists an $\mathbf a_{P_{Y|X=x}} \in \mathcal A_{P_{Y|X=x}}$
 such that 
\begin{align}\label{lemma1e1}
\|\mathbf a_{P_{Y|X=x}}-\mathbf a_{P_Y}\|_2=O(\epsilon).
\end{align} 
\end{lemma}
\begin{proof}
By Assumption \ref{assumption4} and  \eqref{weak-dependent}, for all $x \in \mathcal X$ 
\begin{align}
\sum_{y \in \mathcal Y} \frac{(P_{Y|X}(y|x)-P_Y(y))^2}{P_Y(y)} \leq \epsilon^2, 
\end{align}
where $P_Y(y)>0$ for all $y \in \mathcal Y$. This implies for all $x \in \mathcal X$ and $y \in \mathcal Y$
\begin{align}\label{lemma1e3}
|{P}_{Y|X}(y| x)-{P}_Y(y)|\leq \sqrt{P_Y(y)}\epsilon.
\end{align}
From \eqref{lemma1e3}, we obtain that for all $x \in \mathcal X$
\begin{align}\label{lemma1e4}
\sum_{y \in \mathcal Y} ({P}_{Y|X}(y| x)-{P}_Y(y))^2 \leq \epsilon^2.
\end{align}
Using \eqref{lemma1e4} and Assumption \ref{assumption5}, we get that for any action $\mathbf a_{P_Y} \in \mathcal A_{P_Y}$, there exists an action $\mathbf a_{P_{Y|X=x}} \in \mathcal A_{P_{Y|X=x}}$ such that 
\begin{align}
\|\mathbf a_{P_{Y|X=x}}-\mathbf a_{P_Y}\|_2 =O(\epsilon).
\end{align}
This concludes the proof.
\end{proof}


The Bayes action $\mathbf a_{P_Y}$, as an optimal solution to \eqref{bayes}, is determined only by the marginal distribution $P_Y$ and the loss function $L$. Hence, $\mathbf a_{P_Y}$ is irrelevant of the parameter $\epsilon$ in Assumptions \ref{assumption3}-\ref{assumption4}. Recall that the bias $\tilde{\mathbf b}=[\tilde b_1,\ldots,\tilde b_n]^{\operatorname{T}}\in \mathbb R^n$ satisfies \eqref{BiasVector}.
Hence, the bias $\tilde{\mathbf b}$ is also irrelevant of $\epsilon$. 

Due to Assumption \ref{assumption1}, there exist $\delta > 0$ and $K>0$ such that for all $z\in (\tilde b_i-\delta, \tilde b_i+\delta)
$\begin{align}\label{lemma1_ecompare111}
\frac{|h(z) - h(\tilde b_i)|}{|z-\tilde b_i|} \geq K.  
\end{align}
Hence, if $|z-\tilde b_i| \geq \delta$, then 
\begin{align}\label{lemma1_ecompare}
|h(z) - h(\tilde b_i)| \geq K \delta.
\end{align}
We note that $\delta$ and $K$ depend only on the function $h$ and the bias $\tilde{\mathbf b}$. Hence, $\delta$ and $K$ are irrelevant of $\epsilon$.   

On the other hand, by using \eqref{BiasVector}, Assumption \ref{assumption5}, Assumption \ref{assumption4}, and Lemma \ref{norm_B}, for any  $x \in \mathcal X$ there exists an $\mathbf a_{P_{Y|X=x}}\in \mathcal A_{P_{Y|X=x}}$ that satisfies
\begin{align}\label{bayes_approx}
\|\mathbf a_{P_{Y|X=x}}-\mathbf h(\tilde{\mathbf b})\|_2=O(\epsilon).
\end{align}
In addition, due to Assumption \ref{assumption3}, there exists an optimal solution $(\mathbf f, \mathbf W, \mathbf b)$ to \eqref{reformed_problem} such that
\begin{align}\label{action_approx}
\|\mathbf a_{P_{Y|X=x}}-\mathbf h(\mathbf {{ W}}^{\operatorname{T}}\mathbf { f}(x)+\mathbf { b})\|_2=O(\epsilon).
\end{align}
Combining \eqref{bayes_approx} and \eqref{action_approx}, yields
\begin{align}\label{lemma1e6}
&\|\mathbf h(\tilde{\mathbf b}) -\mathbf h(\mathbf {{ W}}^{\operatorname{T}}\mathbf { f}(x)+\mathbf { b})\|_2\nonumber\\
=&\|\mathbf h(\tilde{\mathbf b})-\mathbf a_{P_{Y|X=x}}+\mathbf a_{P_{Y|X=x}}-\mathbf h(\mathbf {{ W}}^{\operatorname{T}}\mathbf { f}(x)+\mathbf { b})\|_2\nonumber\\
\leq& \|\mathbf h(\tilde{\mathbf b})-\mathbf a_{P_{Y|X=x}}\|_2+\|\mathbf a_{P_{Y|X=x}}-\mathbf h(\mathbf {{ W}}^{\operatorname{T}}\mathbf { f}(x)+\mathbf { b})\|_2\nonumber\\
=&O(\epsilon).
\end{align}
%
Hence, for all  $x\in\mathcal X$ and $i=1,2,\ldots,n$
\begin{align}\label{lemma1_finishing}
h(\mathbf{ w}_i^{\operatorname{T}}\mathbf { f}(x)+{ b}_i)-h(\tilde b_i)=O(\epsilon). 
\end{align}

Define $\alpha_i(x)=\mathbf { w}_i^{\operatorname{T}}\mathbf { f}(x)+{ b}_i-\tilde b_i$. 
According to \eqref{lemma1_finishing}, there exists a constant $C >0$ irrelevant of $\epsilon$, such that  
\begin{align}\label{lemma1_finishing11}
|h(\tilde b_i+\alpha_i(x))-h(\tilde b_i)| \leq C\epsilon.
\end{align}
We choose a sufficiently small $\epsilon>0$ such that $0<\epsilon< \frac{K\delta}{C}$, where $K$ and $\delta$ are given by \eqref{lemma1_ecompare111}. Then, \eqref{lemma1_finishing11} leads to
\begin{align}\label{lemma1_finishing2}
|h(\tilde b_i+\alpha_i(x))-h(\tilde b_i)| < K\delta.
\end{align}
By comparing \eqref{lemma1_ecompare} and \eqref{lemma1_finishing2}, it follows  that $|\alpha_i(x)| < \delta$. Then, by invoking  \eqref{lemma1_ecompare111} again, we can get
\begin{align}
\frac{|h(\tilde b_i+\alpha_i(x))-h(\tilde b_i)|}{|\alpha_i(x)|} \geq K.
\end{align}
Hence, 
\begin{align}\label{eq_alpha_i_bound}
|\alpha_i(x)| \leq& \frac{|h(\tilde b_i+\alpha_i(x))-h(\tilde b_i)|}{K}\nonumber\\
\leq &\frac{C\epsilon}{K}.
\end{align}
This implies $\alpha_i(x)=O(\epsilon)$  for all $x\in\mathcal X$ and $i=1,\ldots,n$. This completes the proof of Lemma \ref{lemma1}.

\ignore{

\subsection{Case 2: Activation Function \eqref{case2} of the Output Layer}
We now consider the activation function in $\eqref{case2}$ and prove the claimed result in three steps: 

\emph{Step 1: We will find an appropriate bias $\hat{\mathbf b}=[\hat b_1,\ldots,\hat b_n]^{\operatorname{T}}$ that satisfies $\mathbf h(\hat{\mathbf b})=\mathbf a_{P_Y}$.} Similar to Case 1, because $\mathcal A\subseteq \mathcal H$ and $\mathbf a_{P_Y}\in\mathcal A$, there exists a bias $\tilde{\mathbf b}=[\tilde b_1,\ldots,\tilde b_n]^{\operatorname{T}}$ such that 
\begin{align}\label{BiasVector1}
\mathbf h(\tilde{\mathbf b})=\left[ \frac{g(\tilde b_1)}{\sum_{j=1}^n g(\tilde b_j)},\ldots,\frac{g(\tilde b_n)}{\sum_{j=1}^n g(\tilde b_j)}\right]^{\operatorname{T}}=\mathbf a_{P_Y}.
\end{align}
We note that the choice of bias $\tilde{\mathbf b}$ in \eqref{BiasVector1} is  irrelevant of $\epsilon$. 
By substituting \eqref{BiasVector1} into \eqref{lemma1e6}, we can obtain that for all $i=1,2,\ldots,n$
\begin{align}\label{lemma1_finishing3}
\frac{g(\mathbf { w}_i^{\operatorname{T}}\mathbf { f}(x)+{ b}_i)}{\sum_{j=1}^n g(\mathbf { w}_j^{\operatorname{T}}\mathbf { f}(x)+{ b}_j)}-\frac{g(\tilde b_i)}{\sum_{j=1}^n g(\tilde b_j)}=O(\epsilon). 
\end{align}
The bias $\tilde{\mathbf b}$ that satisfies \eqref{BiasVector1} and \eqref{lemma1_finishing3} is non-unique. In the sequence, we will find a bias that not only satisfies \eqref{BiasVector1} and \eqref{lemma1_finishing3}, but also has a desirable orthogonality property \eqref{eq_orthogonal}. Denote
\begin{align}
\mathbf{g}&=[g(\tilde b_1), \ldots,g(\tilde b_n)]^{\operatorname{T}},\\
\mathbf{r}&=[g(\mathbf { w}_1^{\operatorname{T}}\mathbf { f}(x)+{ b}_1), \ldots,g(\mathbf { w}_n^{\operatorname{T}}\mathbf { f}(x)+{ b}_n)]^{\operatorname{T}},
\end{align}
and let $v = \mathbf{g}^{\operatorname{T}}\mathbf{r}/\mathbf{g}^{\operatorname{T}} \mathbf g>0$.  
Because the image set of  function $g$ is $[0,\infty)$, for the $v >0$ and  $\tilde{\mathbf b}$ given above, there exists a bias  $\mathbf{\hat b}=[\hat b_1,\ldots,\hat b_n]^{\operatorname{T}}$  that satisfies $g(\hat b_i) = vg(\tilde b_i)$ for all $i=1,2,\ldots,n$ and
\begin{align}\label{eq_scaled}
\frac{g(\hat b_i)}{\sum_{j=1}^n g(\hat b_j)} = \frac{v g(\tilde b_i)}{\sum_{j=1}^n v g(\tilde b_j)}  = \frac{ g(\tilde b_i)}{\sum_{j=1}^n  g(\tilde b_j)}. 
\end{align}
Let us denote
\begin{align}
\mathbf{\hat g}&=[g(\hat b_1), \ldots,g(\hat b_n)]^{\operatorname{T}}, 
\end{align}
then $\mathbf{\hat g} = v \mathbf{g}$ and
\begin{align}\label{eq_orthogonal}
&(\mathbf{r} - \mathbf{\hat g})^{\operatorname{T}}\mathbf{\hat g}\nonumber\\
=&(\mathbf{r} - v\mathbf{g})^{\operatorname{T}}(v\mathbf{g})\nonumber\\
=&v \left(\mathbf{r}^{\operatorname{T}}\mathbf{g} - \frac{\mathbf{g}^{\operatorname{T}}\mathbf{r}}{\mathbf{g}^{\operatorname{T}} \mathbf g} \mathbf{g}^{\operatorname{T}} \mathbf g\right)\nonumber\\
=&0.
\end{align}
By substituting \eqref{eq_scaled} into \eqref{BiasVector1} and \eqref{lemma1_finishing3}, the bias $\hat{\mathbf b}$ satisfies \eqref{eq_orthogonal} and
\begin{align}\label{eq_same_action}
\mathbf h(\hat{\mathbf b})=\left[ \frac{g(\hat b_1)}{\sum_{j=1}^n g(\hat b_j)},\ldots,\frac{g(\hat b_n)}{\sum_{j=1}^n g(\hat b_j)}\right]^{\operatorname{T}}=\mathbf a_{P_Y},\\
\frac{g(\mathbf { w}_i^{\operatorname{T}}\mathbf { f}(x)+{ b}_i)}{\sum_{j=1}^n g(\mathbf { w}_j^{\operatorname{T}}\mathbf { f}(x)+{ b}_j)}-\frac{g(\hat b_i)}{\sum_{j=1}^n g(\hat b_j)}=O(\epsilon). \label{lemma1_finishing1}
\end{align}
By this, we have found a desirable bias $\hat{\mathbf b}$.

\emph{Step 2: We will use \eqref{eq_orthogonal} and \eqref{lemma1_finishing1} to show $\mathbf{r} - \mathbf{\hat g}=O(\epsilon \mathbf 1)$.} Define $\mathbf a(\mathbf z) = [a_1(\mathbf z),\ldots,a_n(\mathbf z)]^{\operatorname{T}}$, where
\begin{align}
a_i(\mathbf z)= \frac{z_i}{\sum_{j=1}^n z_j}.
\end{align}
Then, $\mathbf h(\hat{\mathbf b})=\mathbf a(\mathbf{\hat g})$ and $\mathbf h({\mathbf{W}}^{\operatorname{T}} {\mathbf f}  + {\mathbf b})=\mathbf a(\mathbf{ r})$. Let $\mathbf J_{\mathbf a}(\mathbf z) = \frac{\partial \mathbf a(\mathbf z)}{\partial \mathbf z^{\operatorname{T}}}$ be the Jacobian matrix of $\mathbf a(\mathbf z)$. The $(i,j)$th element of $\mathbf J_{\mathbf a}(\mathbf z)$ is 
\begin{align}
\left(\mathbf J_{\mathbf a}(\mathbf z)\right)_{i,j} = \frac{\partial a_i}{\partial z_j}=\left\{\begin{array}{l l}
\frac{\sum_{k=1}^n z_k - z_i}{\left(\sum_{k=1}^n z_k\right)^2}, & \text{ if } i=j\\
-\frac{z_i}{\left(\sum_{k=1}^n z_k\right)^2}, & \text{ if } i\neq j.
\end{array}\right.
\end{align}
It is easy to check  $\mathbf J_{\mathbf a}(\mathbf z) \mathbf z = \mathbf 0$. Hence, the directional gradient of $\mathbf a(\mathbf z)$ on the direction $\mathbf z$ is 
$\nabla_\mathbf z(\mathbf a(\mathbf z))=\mathbf J_{\mathbf a}(\mathbf z) \mathbf z/\|\mathbf z\|_2=\mathbf 0$. 
The largest and smallest singular values of $\mathbf J_{\mathbf a}(\mathbf z)$ are $\sigma_1(\mathbf J_{\mathbf a}(\mathbf z)) = \sqrt{n}\|z\|_2/\left(\sum_{k=1}^n z_k\right)^2$ and $\sigma_n(\mathbf J_{\mathbf a}(\mathbf z)) = 0$, respectively. If $n\geq 3$, the other singular values are $\sigma_2(\mathbf J_{\mathbf a}(\mathbf z))=\ldots=\sigma_{n-1}(\mathbf J_{\mathbf a}(\mathbf z)) = 1/\sum_{j=1}^n z_j$. The left and right singular vectors of $\mathbf J_{\mathbf a}(\mathbf z)$ associated to the smallest singular value $\sigma_n(\mathbf J_{\mathbf a}(\mathbf z))=0$ are $\mathbf 1/\|\mathbf 1\|_2$ and $\mathbf z/\|\mathbf z\|_2$, respectively. 

The directional gradient of $\mathbf a(\mathbf z)$ at the point $\mathbf z=\mathbf{\hat g}$ on the direction $(\mathbf{r} - \mathbf{\hat g})$ is 
\begin{align}
\nabla_{\mathbf{r} - \mathbf{\hat g}}(\mathbf a(\mathbf z))\big|_{\mathbf z=\mathbf{\hat g}}=\frac{\mathbf J_{\mathbf a}(\mathbf{\hat g}) (\mathbf{r} - \mathbf{\hat g})}{\|\mathbf{r} - \mathbf{\hat g}\|_2}.
\end{align}
Notice that $\mathbf{\hat g}/\|\mathbf{\hat g}\|_2$ is the right singular vector of $\mathbf J_{\mathbf a}(\mathbf {\hat g})$ associated to the smallest singular value $\sigma_n(\mathbf J_{\mathbf a}(\mathbf {\hat g}))=0$.
Taking \eqref{eq_orthogonal} into consideration, we can get  
\begin{align}
&\nabla_{\mathbf{r} - \mathbf{\hat g}}(\mathbf a(\mathbf z))\big|_{\mathbf z=\mathbf{\hat g}} 
\geq \min_{\mathbf z: \mathbf z^{\operatorname{T}}\mathbf{\hat g}=0}\frac{\mathbf J_{\mathbf a}(\mathbf{\hat g}) \mathbf z}{\|\mathbf z\|_2}
=\sigma_{n-1}(\mathbf J_{\mathbf a}(\mathbf {\hat g}))
\geq \frac{1}{\|\mathbf {\hat g}\|_1}, 
\end{align}
where $\sigma_{n-1}(\mathbf J_{\mathbf a}(\mathbf {\hat g}))$ is the second smallest singular value of $\mathbf J_{\mathbf a}(\mathbf {\hat g})$. 

The directional gradient $\nabla \mathbf h(\mathbf{\hat b};\bm\alpha)$ of $\mathbf h({\mathbf b})$ at the point ${\mathbf b}=\mathbf{\hat b}$ on the direction $\bm\alpha$ is
\begin{align}
\nabla \mathbf h(\mathbf{\hat b};\bm\alpha) = \frac{\mathbf J_{\mathbf a}(\mathbf z(\mathbf{\hat b})) \mathbf J_{\mathbf z}(\mathbf {\hat b}) \bm\alpha}{\|\bm\alpha\|_2}.
\end{align}

Denote $\alpha_i=\mathbf { w}_i^{\operatorname{T}}\mathbf { f}(x)+{ b}_i-\hat b_i$ and $\bm \alpha = [\alpha_1,\ldots,\alpha_n]^{\operatorname{T}}$. 

In the following, we will use Assumption \ref{assumption1} and \eqref{lemma1_finishing1} to show $\bm\alpha=O(\epsilon \mathbf 1)$. }

\section{Proof of Lemma \ref{lemma2}}\label{plemma2}
Let us define big-O and little-o notations for vectors and matrices, which will be used in the proof. 

\begin{definition}[Big-O and Little-o Notations for Vectors] 
Consider two vector functions $\mathbf f: \mathbb{R}\mapsto \mathbb{R}^{n}$ and $\mathbf g: \mathbb{R}\mapsto \mathbb{R}^{n}$. We say $\mathbf f(x) = O(\mathbf g(x))$, if there exist constants $M > 0$ and $d > 0$ such that 
\begin{align}
\|\mathbf f(x)\|_2 \leq M \|\mathbf g(x)\|_2,  ~ \text{for all }x\text{ with }|x| < d,
\end{align}
where $\|\mathbf f\|_2= (\sum_{i=1}^n f_i^2)^{1/2}$ is the $l_2$ norm of vector $\mathbf f$. 
We say $\mathbf f(x) = o(\mathbf g(x))$, if for each $M > 0$ there exists a real number $d > 0$ such that 
\begin{align}\label{eq_little_o}
\|\mathbf f(x)\|_2 \leq M \|\mathbf g(x)\|_2,  ~ \text{for all }x\text{ with }|x| < d. 
\end{align}
If $\|\mathbf g(x)\|_2\neq 0$, then \eqref{eq_little_o} is equivalent to 
\begin{align}
\lim_{x\rightarrow0} \|\mathbf f(x)\|_2/\|\mathbf g(x)\|_2 = 0.
\end{align}
\end{definition}

\begin{definition}[Big-O and Little-o Notations for Matrices] 
Consider two matrix functions $\mathbf F: \mathbb{R}\mapsto \mathbb{R}^{n}\times\mathbb{R}^{n}$ and $\mathbf G: \mathbb{R}\mapsto \mathbb{R}^{n}\times\mathbb{R}^{n}$. We say $\mathbf F(x) = O(\mathbf G(x))$, if there exist constants $M > 0$ and $d > 0$ such that 
\begin{align}
\|\mathbf F(x)\|_2 \leq M \|\mathbf G(x)\|_2,  ~ \text{for all }x\text{ with }|x| < d, 
\end{align}
where $\|\mathbf A\|_2 =\sigma_1 (\mathbf A)$ is the spectral norm of matrix $\mathbf A$. 
In addition, we say $\mathbf F(x) = o(\mathbf G(x))$, if for every $M > 0$ there exists a real number $d > 0$ such that 
\begin{align}\label{eq_little_o_matrix}
\|\mathbf F(x)\|_2 \leq M \|\mathbf G(x)\|_2,  ~ \text{for all }x\text{ with }|x| < d. 
\end{align}
If $\|\mathbf G(x)\|_2\neq 0$, then \eqref{eq_little_o_matrix} is equivalent to 
\begin{align}
\lim_{x\rightarrow0} \|\mathbf F(x)\|_2/\|\mathbf G(x)\|_2 = 0.
\end{align}
\end{definition}

Let $\mathbf M_L(x)$ denote the Hessian matrix  
\begin{align}
\mathbf M_L(x)= \frac{\partial^2 E_{Y \sim P_{Y|X=x}}[L(Y, \mathbf a)]}{\partial \mathbf a \partial \mathbf a^{\operatorname{T}}}\bigg|_{\mathbf a=\mathbf a_{P_{Y|X=x}}},
\end{align}
The $(i, j)$-th element of $\mathbf M_L(x)$ is 
\begin{align}
(\mathbf M_L(x))_{i,j}=\frac{\partial^2 E_{Y \sim P_{Y|X=x}}[L(Y, \mathbf a)]}{\partial a_i \partial a_j}\bigg|_{\mathbf a=\mathbf a_{P_{Y|X=x}}}.
\end{align}

\begin{lemma}\label{Hessian_lemma}
If Assumptions \ref{assumption5}, \ref{assumption2}, and \ref{assumption4} hold, then
\begin{align}\label{hessianlocal}
(\mathbf M_L(x))_{i,j}=(\mathbf M_L)_{i,j}+o(1).
\end{align}
\end{lemma}
\begin{proof}
\ignore{{\red In the next equation, $P_Y$ occurs twice, in the inner product with loss function $L$ and in the Bayes action (I hope you have multivariable calculus book for knowledge). Because the inner product is linear wrt $P_Y$,  continuity wrt $P_Y$ follows. The continuity of $g$ on the action, and the continuity of Bayes action on $P_Y$ (continuity of composite function) are needed. All these should be discussed in the proof.}} 
Consider the function
\begin{align}
g(P_Y)=\frac{\partial^2 \mathbb E_{Y \sim P_Y}[L(Y, \mathbf a)]}{\partial a_i \partial a_j}\bigg|_{\mathbf a=\mathbf a_{P_Y}},
\end{align}
where the Bayes action $\mathbf a_{P_Y}$ satisfies Assumption \ref{assumption5}. 

Because of Assumption \ref{assumption5}, we can say that the $\mathbf a_{P_Y}$ is a continuous function of $P_Y$. In addition, due to Assumption \ref{assumption2} and by using the continuity property of a composite function, we obtain that $g(P_Y)$ is a continuous function of $P_Y$. 

Due to Assumption \ref{assumption4}, we can get that for all $x\in\mathcal X$ and $y\in\mathcal Y$
\begin{align}\label{HL1}
P_{Y|X}(y|x)=P_Y(y)+O(\epsilon).
\end{align}
In addition, because $g$ is continuous, \eqref{HL1} implies 
\begin{align}
g(P_{Y|X=x})=g(P_Y)+o(1).
\end{align}
This concludes the proof.
\end{proof}

It is known that
\begin{align}\label{Lemma1proofeq1}
D_L(\mathbf a_{P_{Y|X=x}}||\mathbf a)\geq 0,
\end{align}
where equality is achieved at $\mathbf a = \mathbf a_{P_{Y|X=x}}$, i.e., 
\begin{align}\label{stationary_point}
D_L(\mathbf a_{P_{Y|X=x}}||\mathbf a_{P_{Y|X=x}})=0.
\end{align}
In addition, the function $\mathbf a \mapsto L(y, \mathbf a)$ is twice differentiable for all $y\in\mathcal Y$. Because of these properties, by taking the second order Taylor series expansion of the function $\mathbf a \mapsto D_L(\mathbf a_{P_{Y|X=x}}||\mathbf a)$ at the point $\mathbf a=\mathbf a_{P_{Y|X=x}},$ we can get
\begin{align}\label{Lemma1proofeq3}
D_L(\mathbf a_{P_{Y|X=x}}||\mathbf a)=&\frac{1}{2}(\mathbf a-\mathbf a_{P_{Y|X=x}})^{\operatorname{T}} \mathbf M_L(x) (\mathbf a-\mathbf a_{P_{Y|X=x}}) \nonumber\\
\!\!\!\!\!&+o(\|\mathbf a-\mathbf a_{P_{Y|X=x}}\|^2_2).
\end{align}
Let $\mathbf a=\mathbf h({{\mathbf W}}^{\operatorname{T}}{\mathbf f}(x)+{\mathbf b})$ in \eqref{Lemma1proofeq3}, we obtain
\begin{align}\label{Lemma1proofeq4}
&D_L(\mathbf a_{P_{Y|X=x}}||\mathbf h({{\mathbf W}}^{\operatorname{T}}{\mathbf f}(x)+{\mathbf b}))\nonumber\\
=&\frac{1}{2}(\mathbf h({{\mathbf W}}^{\operatorname{T}}{\mathbf f}(x)+{\mathbf b})-\mathbf a_{P_{Y|X=x}})^{\operatorname{T}} \mathbf M_L(x) \nonumber\\
& \times (\mathbf h({{\mathbf W}}^{\operatorname{T}}{\mathbf f}(x)+{\mathbf b})-\mathbf a_{P_{Y|X=x}}) \nonumber\\
\!\!\!\!\!&+o\left(\|\mathbf h({{\mathbf W}}^{\operatorname{T}}{\mathbf f}(x)+{\mathbf b})-\mathbf a_{P_{Y|X=x}}\|^2_2\right).
\end{align}
\ignore{{\red Is the following assumption number correct?}}
Due to Assumption \eqref{constraint}, \eqref{Lemma1proofeq4} can be reduced to 
\begin{align}\label{Lemma1proofeq5}
&D_L(\mathbf a_{P_{Y|X=x}}||\mathbf h({{\mathbf W}}^{\operatorname{T}}{\mathbf f}(x)+{\mathbf b}))\nonumber\\
=&\frac{1}{2}(\mathbf h({{\mathbf W}}^{\operatorname{T}}{\mathbf f}(x)+{\mathbf b})-\mathbf a_{P_{Y|X=x}})^{\operatorname{T}} \mathbf M_L(x)  \nonumber\\
& \times (\mathbf h({{\mathbf W}}^{\operatorname{T}}{\mathbf f}(x)+{\mathbf b})-\mathbf a_{P_{Y|X=x}}) +o(\epsilon^2).
\end{align}

\ignore{{\red Which assumption of $\mathbf h$ is needed here?}}
Because $h$ is continuously twice differentiable, we take the first order Taylor series expansion of $\mathbf h(\mathbf b)$ at the point $\mathbf b=\tilde{\mathbf b}$, which yields
\begin{align}\label{activationlocal1}
\mathbf h(\mathbf b)=\mathbf h(\tilde{\mathbf b})+\mathbf J(\mathbf b-\tilde{\mathbf b})+o(\mathbf b-\tilde{\mathbf b}).
\end{align}
\ignore{{\red Note: $\mathbf J$ is not of full rank for the second activation function of the output layer? This is a good reason to abandon the second activation function.}}
In \eqref{activationlocal1}, by using Lemma \ref{lemma1} and letting $\mathbf b={{\mathbf W}}^{\operatorname{T}}{\mathbf f}(x)+{\mathbf b}$, we can get
\begin{align}\label{activationlocal}
&\mathbf h({{\mathbf W}}^{\operatorname{T}}{\mathbf f}(x)+{\mathbf b})\nonumber\\
=&\mathbf h(\tilde{\mathbf b})+\mathbf J{{\mathbf W}}^{\operatorname{T}}{\mathbf f}(x)+\mathbf J {\mathbf d}+o(\epsilon \bm 1),
\end{align}
where ${\mathbf d}={\mathbf b} - \tilde{\mathbf b}$ and $\bm 1=[1, \ldots, 1]^{\operatorname{T}} \in \mathbb R^n$.


Define
\begin{align}
\mathbf q_1&=\mathbf a_{P_{Y|X=x}}-\bm \mu_{\mathbf a}, \\
\mathbf q_2&=\mathbf J{{\mathbf W}}^{\operatorname{T}}({\mathbf f}(x)-{\bm \mu}_{\mathbf f}),\\
\mathbf q_3&=\mathbf a_{P_Y}-\bm \mu_{\mathbf a}+\mathbf J {\mathbf d}+\mathbf J{{\mathbf W}}^{\operatorname{T}}{\bm \mu}_{\mathbf f}.
\end{align}

By using \eqref{constraint} and Lemma \ref{norm_B}, we get
\begin{align}\label{q1q2q3}
&\|\mathbf q_1-\mathbf q_2-\mathbf q_3\|_2 \nonumber\\
=&\|\mathbf a_{P_{Y|X=x}}-\mathbf a_{P_Y}-\mathbf J({{\mathbf W}}^{\operatorname{T}}{\mathbf f}(x)+{\mathbf d})\|_2\nonumber\\
\leq&\|\mathbf a_{P_{Y|X=x}}-\mathbf a_{P_Y}\|_2+\|\mathbf J({{\mathbf W}}^{\operatorname{T}}{\mathbf f}(x)+{\mathbf d})\|_2 \nonumber\\
\leq&\|\mathbf a_{P_{Y|X=x}}-\mathbf a_{P_Y}\|_2+\|\mathbf J\|_2\|({{\mathbf W}}^{\operatorname{T}}{\mathbf f}(x)+{\mathbf d})\|_2 \nonumber\\
\leq&\|\mathbf a_{P_{Y|X=x}}-\mathbf a_{P_Y}\|_2+\sigma_{\text{max}}(\mathbf J)\|({{\mathbf W}}^{\operatorname{T}}{\mathbf f}(x)+{\mathbf d})\|_2 \nonumber\\
=&O(\epsilon),
\end{align}
where $\sigma_{\text{max}}(\mathbf J)=\max_{i}h'(\tilde b_i)$. 

Substituting \eqref{hessianlocal} and \eqref{activationlocal} to \eqref{Lemma1proofeq5}, we obtain
\begin{align}
\!\!\!&D_L(\mathbf a_{P_{Y|X=x}}||\mathbf h({{\mathbf W}}^{\operatorname{T}}{\mathbf f}(x)+{\mathbf b}))\nonumber\\
=&\frac{1}{2}(\mathbf q_1-\mathbf q_2-\mathbf q_3+o(\epsilon\bm 1))^{\operatorname{T}} \nonumber\\
&\times (\mathbf M_L+o(\mathbf I))(\mathbf q_1-\mathbf q_2-\mathbf q_3+o( \epsilon \bm 1))\nonumber\\
=&\frac{1}{2}\bigg[(\mathbf q_1-\mathbf q_2-\mathbf q_3)^{\operatorname{T}}\mathbf M_L(\mathbf q_1-\mathbf q_2-\mathbf q_3)\nonumber\\
&~~~+2(\mathbf q_1-\mathbf q_2-\mathbf q_3)^{\operatorname{T}}\mathbf M_L ~o(\epsilon \bm 1)\nonumber\\
&~~~+2(\mathbf q_1-\mathbf q_2-\mathbf q_3)^{\operatorname{T}}o(\mathbf I)~o(\epsilon \bm 1)\nonumber\\
&~~~+o(\epsilon \bm 1^{\operatorname{T}})~(\mathbf M_L+o(\mathbf I))~o(\epsilon \bm 1)\nonumber\\
&~~~+(\mathbf q_1-\mathbf q_2-\mathbf q_3)^{\operatorname{T}}o(\mathbf I) (\mathbf q_1-\mathbf q_2-\mathbf q_3)\bigg]+o(\epsilon^2).\!\!
\end{align}

Using \eqref{q1q2q3}, we can write
\begin{align}\label{L2eImp}
&D_L(\mathbf a_{P_{Y|X=x}}||\mathbf h({{\mathbf W}}^{\operatorname{T}}{\mathbf f}(x)+{\mathbf b}))\nonumber\\
=&\frac{1}{2}(\mathbf q_1-\mathbf q_2-\mathbf q_3)^{\operatorname{T}} \mathbf M_L(\mathbf q_1-\mathbf q_2-\mathbf q_3)+o(\epsilon^2).
\end{align}
Because $\mathbf M_L=\mathbf R_L^{\operatorname{T}}\mathbf R_L$, we get
\begin{align}
&\!\!\!\!\!\!\!D_L(\mathbf a_{P_{Y|X=x}}||\mathbf h({{\mathbf W}}^{\operatorname{T}}{\mathbf f}(x)+{\mathbf b}))\nonumber\\
&\!\!\!\!\!\!\!\!=\frac{1}{2}(\mathbf R_L(\mathbf q_1\!-\!\mathbf q_2\!-\!\mathbf q_3))^{\operatorname{T}} (\mathbf R_L(\mathbf q_1\!-\!\mathbf q_2\!-\!\mathbf q_3))\!+\!o(\epsilon^2).\!\!
\end{align}
Multiply the above equation by $P_X(x)$, yields
\begin{align}\label{lastdecomposed}
& P_X(x) D_L(\mathbf a_{P_{Y|X=x}}||\mathbf h({{\mathbf W}}^{\operatorname{T}}{\mathbf f}(x)+{\mathbf b}))\nonumber\\
=&\frac{1}{2}(\mathbf R_L\sqrt{P_X(x)}(\mathbf q_1-\mathbf q_2-\mathbf q_3))^{\operatorname{T}} \nonumber\\
&~~~~~~~~~~\times (\mathbf R_L\sqrt{P_X(x)}(\mathbf q_1-\mathbf q_2-\mathbf q_3))+o(\epsilon^2)\nonumber\\
=&\frac{1}{2}(\mathbf R_L\sqrt{P_X(x)}(\mathbf q_1-\mathbf q_2))^{\operatorname{T}}(\mathbf R_L\sqrt{P_X(x)}(\mathbf q_1-\mathbf q_2))\nonumber\\
&-P_X(x)(\mathbf q_1-\mathbf q_2)^{\operatorname{T}}\mathbf M_L\mathbf q_3+\frac{1}{2}P_X(x)\mathbf q_3^{\operatorname{T}}\mathbf M_L\mathbf q_3+o(\epsilon^2).
\end{align}
By substituting \eqref{matrixf}-\eqref{matrixJ} into \eqref{lastdecomposed} and taking the summation over $x \in \mathcal X$, we derive 
\begin{align}
&\sum_{x \in \mathcal X} P_X(x) D_L(\mathbf a_{P_{Y|X=x}}||\mathbf h({{\mathbf W}}^{\operatorname{T}}{\mathbf f}(x)+{\mathbf b}))\nonumber\\
=&\frac{1}{2} \|\mathbf{R}_L\mathbf{B}-\mathbf{R}_L\mathbf{J}{\mathbf{W}}^{\operatorname{T}} {\mathbf{\Xi}}_{\mathbf f}\|_{F}^2\nonumber\\
&+\frac{1}{2}\sum_{x \in \mathcal X}P_X(x)(\mathbf q_3-2\mathbf q_1+2\mathbf q_2)^{\operatorname{T}}\mathbf M_L\mathbf q_3+o(\epsilon^2) \nonumber\\
=&\frac{1}{2} \|\mathbf{R}_L\mathbf{B}-\mathbf{R}_L\mathbf{J}{\mathbf{W}}^{\operatorname{T}} {\mathbf{\Xi}}_{\mathbf f}\|_{F}^2\nonumber\\
&+\frac{1}{2}\bigg(\mathbf q_3-2\sum_{x \in \mathcal X}P_X(x) \mathbf q_1+2 \sum_{x \in \mathcal X}P_X(x)\mathbf q_2\bigg)^{\operatorname{T}}\mathbf M_L {\mathbf q}_3 \nonumber\\
&+o(\epsilon^2) \nonumber\\
=&\frac{1}{2} \|\mathbf{R}_L\mathbf{B}-\mathbf{R}_L\mathbf{J}{\mathbf{W}}^{\operatorname{T}} {\mathbf{\Xi}}_{\mathbf f}\|_{F}^2+\frac{1}{2}\mathbf q_3^{\operatorname{T}}\mathbf M_L{\mathbf q}_3 +o(\epsilon^2) \nonumber\\
=&\frac{1}{2} \|\mathbf{R}_L\mathbf{B}-\mathbf{R}_L\mathbf{J}{\mathbf{W}}^{\operatorname{T}} {\mathbf{\Xi}}_{\mathbf f}\|_{F}^2+\frac{1}{2}\eta({\mathbf d}, {\mathbf f})+o(\epsilon^2),
\end{align}
where the second equality holds because  $\mathbf q_3$ and $\mathbf{M}_L$ do not change with respect to $x$, and the third equality holds because $\sum_{x \in \mathcal X}P_X(x)\mathbf q_1=\bm 0$ and $\sum_{x \in \mathcal X}P_X(x)\mathbf q_2=\bm 0$. This completes the proof.
\section{Proof of Theorem \ref{theorem1}}\label{ptheorem1}
Notice that $\mathbf d$ only affects the second term of \eqref{approximate_problem}. To optimize $\mathbf d$, we take the derivative 
\begin{align}
\frac{\partial \eta(\mathbf d, \mathbf f)}{\partial \mathbf d}=2\mathbf M_L(\mathbf J \mathbf d+\mathbf J\mathbf W^{\operatorname{T}}\bm \mu_{\mathbf f}+\mathbf a_{P_Y}-\bm \mu_{\mathbf a}).
\end{align}
Equating the derivative to zero, we get \eqref{optimal_bias}. Substituting the optimal bias into \eqref{eta}, we get
\begin{align}
\eta(\mathbf d, \mathbf f)=0,
\end{align}
which is the minimum value of the function $\eta(\mathbf d, \mathbf f)$.

Next, for fixed $\mathbf \Xi_{\mathbf f}$, we need to optimize $\mathbf \Xi_{\mathbf W}$ by solving 
\begin{align}
\min_{\mathbf \Xi_{\mathbf W}}\|\tilde{\mathbf B}-\mathbf \Xi_{\mathbf W} \mathbf \Xi_{\mathbf f}\|_{F}^2,
\end{align}
which is a convex optimization problem. By setting the derivative 
\begin{align}
\frac{\partial}{\partial \mathbf \Xi_{\mathbf W}}\|\tilde{\mathbf B}-\mathbf \Xi_{\mathbf W} \mathbf \Xi_{\mathbf f}\|_{F}^2=2(\mathbf \Xi_{\mathbf W}\mathbf \Xi_{\mathbf f}\mathbf \Xi_{\mathbf f}^{\operatorname{T}}-\tilde{\mathbf B}\mathbf {\Xi}_{\mathbf f}^{\operatorname{T}})
\end{align}
to zero, we find the optimal solution 
\begin{align}
\mathbf \Xi_{\mathbf W}^*=\mathbf R_L \mathbf B \mathbf {\Xi}_{\mathbf f}^{\operatorname{T}}(\mathbf \Xi_{\mathbf f}\mathbf \Xi_{\mathbf f}^{\operatorname{T}})^{\operatorname{-1}}.
\end{align}


\section{Proof of Theorem \ref{theorem2}}\label{ptheorem2}
To optimize $\mathbf \Xi_{\mathbf f}$, we set
\begin{align}
\frac{\partial}{\partial \mathbf \Xi_{\mathbf f}}\|\tilde{\mathbf B}-\mathbf \Xi_{\mathbf W} \mathbf \Xi_{\mathbf f}\|_{F}^2=2(\mathbf \Xi_{\mathbf f}^{\operatorname{T}}\mathbf \Xi_{\mathbf W}^{\operatorname{T}}\mathbf \Xi_{\mathbf W}-\tilde{\mathbf B}^{\operatorname{T}}\mathbf \Xi_{\mathbf W})
\end{align}
to zero and get 
\begin{align}
\mathbf \Xi_{\mathbf f}^*=(\mathbf \Xi_{\mathbf W}^{\operatorname{T}}\mathbf \Xi_{\mathbf W})^{\operatorname{-1}}\mathbf \Xi_{\mathbf W}^{\operatorname{T}}\tilde{\mathbf B},
\end{align}
where
\begin{align}
\mathbf \Xi_{\mathbf f}^*\sqrt{\mathbf p_X}=\mathbf 0,
\end{align}
where the vector $\sqrt{\mathbf p_X}$ is defined in \eqref{px}.

Because $\bm \mu_{\mathbf f}$ only affects the second term of \eqref{approximate_problem}, we set the derivative 
\begin{align}
&\frac{\partial}{\partial \bm \mu_{\mathbf f}}\eta(\mathbf d, \mathbf f) \nonumber\\
=&2\mathbf \Xi_{\mathbf W} \mathbf R_L(\mathbf a_{P_Y}-\bm \mu_{\mathbf a}+\mathbf J\mathbf d)+2 \mathbf \Xi_{\mathbf W}^{\operatorname{T}}\mathbf \Xi_{\mathbf W} \bm \mu_{\mathbf f}
\end{align}
to zero and obtain \eqref{optimal_mean}.

\section{Proof of Theorem \ref{theorem3}}\label{ptheorem3}
One lower bound of the first term in \eqref{approximate_problem} is given by
\begin{align}\label{theorem2e1}
\|\mathbf{\tilde B}- \mathbf \Xi_{\mathbf W} \mathbf{\Xi}_{\mathbf f}\|_{F}^2\geq \sum_{i=k+1}^{K} \sigma_i^2.
\end{align}

By using Eckart–Young–Mirsky Theorem \cite{eckart1936approximation}, if we substitute the value of $\mathbf{\Xi}_{\mathbf f}^*$ and $\mathbf \Xi_{\mathbf W}^*$ from \eqref{optimal_pair1} into \eqref{theorem2e1}, equality with the lower bound is achieved in \eqref{theorem2e1}.  

If the optimal bias $\mathbf d^*$ and the optimal mean $\bm \mu_{\mathbf f}^*$ satisfy \eqref{optimal_bias_mean}, we get the minimum of $\eta(\mathbf d, \mathbf f)$.

\ignore{

\section{Proof of Lemma \ref{local_approximation_2}}\label{plocal_approximation_2}
{\red Instead of saying ``we can do sth,'' use ``we do sth.''} 
We take the first order Taylor series expansion of $\mathbf h(\mathbf z)$ at the point $\mathbf z=\hat{\mathbf b}$, which yields
\begin{align}\label{T4e1}
\mathbf h(\mathbf z)=\mathbf h(\hat{\mathbf b})+\tilde{\mathbf J}(\mathbf z-\hat{\mathbf b})+o(\mathbf z-\hat{\mathbf b}).
\end{align}
{\red Notational confusion between $\tilde{\mathbf J}$ and $\hat{\mathbf J}$. The second seems better, as defined in (45).}
Let $\mathbf z={\mathbf{W}}^{\operatorname{T}}\bm \psi({{\mathbf W}^{(1){\operatorname{T}}}}{\mathbf f}^{(1)}(x)+{\mathbf b}^{(1)})+\mathbf b$. Then, by using \eqref{constraint1}, we obtain
\begin{align}\label{T4e2}
&\mathbf h({\mathbf{W}}^{\operatorname{T}}\bm \psi({{\mathbf W}^{(1){\operatorname{T}}}}{\mathbf f}^{(1)}(x)+{\mathbf b}^{(1)})+\mathbf b)\nonumber\\
=&\mathbf h(\hat{\mathbf b})+\tilde{\mathbf J}\left({\mathbf{W}}^{\operatorname{T}}\bm \psi({{\mathbf W}^{(1){\operatorname{T}}}}{\mathbf f}^{(1)}(x)+{\mathbf b}^{(1)})+\mathbf b-\hat{\mathbf b}\right)\nonumber\\
&+o({\mathbf{W}}^{\operatorname{T}}\bm \psi({{\mathbf W}^{(1){\operatorname{T}}}}{\mathbf f}^{(1)}(x)+{\mathbf b}^{(1)})+\mathbf b-\hat{\mathbf b})\nonumber\\
=&\mathbf h(\hat{\mathbf b})+\tilde{\mathbf J}\left({\mathbf{W}}^{\operatorname{T}}\bm \psi({{\mathbf W}^{(1){\operatorname{T}}}}{\mathbf f}^{(1)}(x)+{\mathbf b}^{(1)})+\mathbf b-\hat{\mathbf b}\right)+o(\epsilon \mathbf 1).
\end{align}
Next, we take the first order Taylor series expansion of $\bm \psi(\mathbf z^{(1)})$ at the point $\mathbf z^{(1)}=\hat{\mathbf b}^{(1)}$ and get
\begin{align}\label{T4e3}
\!\!\bm \psi(\mathbf z^{(1)})=\bm \psi(\hat{\mathbf b}^{(1)})+\mathbf J_1(\mathbf z^{(1)}-\hat{\mathbf b}^{(1)})+o(\mathbf z^{(1)}-\hat{\mathbf b}^{(1)}).\!\!
\end{align}
Let $\mathbf z^{(1)}={{\mathbf W}^{(1){\operatorname{T}}}}{\mathbf f}^{(1)}(x)+{\mathbf b}^{(1)}$, then by using \eqref{constraint2}, we get
\begin{align}\label{T4e4}
&\bm \psi({{\mathbf W}^{(1){\operatorname{T}}}}{\mathbf f}^{(1)}(x)+{\mathbf b}^{(1)})\nonumber\\
=&\bm \psi(\hat{\mathbf b}^{(1)})+\mathbf J_1\left({{\mathbf W}^{(1){\operatorname{T}}}}{\mathbf f}^{(1)}(x)+{\mathbf b}^{(1)}-\hat{\mathbf b}^{(1)}\right)+o(\epsilon\mathbf1).
\end{align}
{\red How to ensure \eqref{T4e5}??}
Now, if we choose $\mathbf z=\mathbf W^{\operatorname{T}}\bm \psi(\hat{\mathbf b}^{(1)})+\hat {\mathbf b}$ in \eqref{T4e1} such that 
\begin{align}\label{T4e5}
\mathbf W^{\operatorname{T}}\bm \psi(\hat{\mathbf b}^{(1)})=o(\epsilon\mathbf 1),
\end{align}
we get 
\begin{align}\label{T4e6}
\!\!\mathbf h(\mathbf W^{\operatorname{T}}\bm \psi(\hat{\mathbf b}^{(1)})+\hat {\mathbf b})=\mathbf h(\hat{\mathbf b})+\tilde{\mathbf J}(\mathbf W^{\operatorname{T}}\bm \psi(\hat{\mathbf b}^{(1)}))+o(\epsilon\mathbf 1).\!\!
\end{align}
Substituting \eqref{bias2}, \eqref{T4e4}, and \eqref{T4e6} into \eqref{T4e2}, we get
\begin{align}\label{T4e7}
&\mathbf h({\mathbf{W}}^{\operatorname{T}}\bm \psi({{\mathbf W}^{(1){\operatorname{T}}}}{\mathbf f}^{(1)}(x)+{\mathbf b}^{(1)})+\mathbf b)\nonumber\\
=&\mathbf a_{P_Y}+\tilde{\mathbf J}\left({\mathbf{W}}^{\operatorname{T}}\mathbf J_1\left({{\mathbf W}^{(1){\operatorname{T}}}}{\mathbf f}^{(1)}(x)+{\mathbf b}^{(1)}-\hat{\mathbf b}^{(1)}\right)+\mathbf b-\hat{\mathbf b}\right)\nonumber\\
&+o(\epsilon \mathbf 1).
\end{align}
{\red Some details are skipped in the above step.}

Define
\begin{align}\label{T4e8}
\mathbf q_1^{(1)}=&\mathbf a_{P_{Y|X=x}}-\bm \mu_{\mathbf a},\\
\mathbf q_2^{(1)}=&\tilde{\mathbf J}{{\mathbf W}}^{\operatorname{T}}\mathbf J_1{{\mathbf W}^{(1){\operatorname{T}}}}({\mathbf f}^{(1)}(x)-{\bm \mu}_{\mathbf f^{(1)}}), \\
\mathbf q_3^{(1)}=&\mathbf a_{P_Y}-\bm \mu_{\mathbf a}+\mathbf J {\mathbf c}+\mathbf J{{\mathbf W}}^{\operatorname{T}}\mathbf J_1{{\mathbf W}^{(1){\operatorname{T}}}}{\bm \mu}_{\mathbf f^{(1)}}.
\end{align}
By using \eqref{norm_B}, \eqref{constraint1}, and \eqref{constraint2}, we get 
\begin{align}
&\|\mathbf q^{(1)}_1-\mathbf q^{(1)}_2-\mathbf q^{(1)}_3\|_2 \nonumber\\
=&\|\mathbf a_{P_{Y|X=x}}-\mathbf a_{P_Y}\nonumber\\
&+\tilde{\mathbf J}\left({\mathbf{W}}^{\operatorname{T}}\mathbf J_1\left({{\mathbf W}^{(1){\operatorname{T}}}}{\mathbf f}^{(1)}(x)+{\mathbf b}^{(1)}-\hat{\mathbf b}^{(1)}\right)+\mathbf b-\hat{\mathbf b}\right)\|_2\nonumber\\
\leq& \|\mathbf a_{P_{Y|X=x}}-\mathbf a_{P_Y}\|_2 \nonumber\\
&+\|\tilde{\mathbf J}\left({\mathbf{W}}^{\operatorname{T}}\mathbf J_1\left({{\mathbf W}^{(1){\operatorname{T}}}}{\mathbf f}^{(1)}(x)+{\mathbf b}^{(1)}-\hat{\mathbf b}^{(1)}\right)+\mathbf b-\hat{\mathbf b}\right)\|_2 \nonumber\\
\leq &O(\epsilon)+\big\|\tilde{\mathbf J}\big\|_2 \nonumber\\
&\times \big\|\left({\mathbf{W}}^{\operatorname{T}}\mathbf J_1\left({{\mathbf W}^{(1){\operatorname{T}}}}{\mathbf f}^{(1)}(x)+{\mathbf b}^{(1)}-\hat{\mathbf b}^{(1)}\right)+\mathbf b-\hat{\mathbf b}\right)\big\|_2\nonumber\\
=&O(\epsilon).
\end{align}
{\red What is the order of $\hat{\mathbf J}$? Any proof?}
The last equality holds because \eqref{constraint1}, \eqref{constraint2}, and \eqref{T4e5} imply
\begin{align}
&\big\|\left({\mathbf{W}}^{\operatorname{T}}\mathbf J_1\left({{\mathbf W}^{(1){\operatorname{T}}}}{\mathbf f}^{(1)}(x)+{\mathbf b}^{(1)}-\hat{\mathbf b}^{(1)}\right)+\mathbf b-\hat{\mathbf b}\right)\big\|_2 \nonumber\\
=&O(\epsilon).
\end{align}

In addition, using \eqref{T4e7} yields
\begin{align}\label{T4e9}
&D_L(\mathbf a_{P_{Y|X=x}} || \mathbf h(\mathbf{W}^{\operatorname{T}}\bm \psi({\mathbf W}^{(1)\operatorname{T}} {\mathbf f}^{(1)}(x)+{\mathbf b}^{(1)})+\mathbf b)) \nonumber\\
=&\frac{1}{2}(\mathbf R_L(\mathbf q^{(1)}_1-\mathbf q^{(1)}_2-\mathbf q^{(1)}_3))^{\operatorname{T}}(\mathbf R_L(\mathbf q^{(1)}_1-\mathbf q^{(1)}_2-\mathbf q^{(1)}_3)).
\end{align}

Now, we will same computation as described in Appendix \ref{plemma2}. First, multiply the above equation by $P_X(x)$ and then, taking the summation over $x \in \mathcal X$, we get \eqref{local_approximation_2}. {\red What is \eqref{local_approximation_2}?}
\section{Proof of Theorem \ref{theorem4}}\label{ptheorem4}

For fixed $\mathbf \Xi_{\mathbf f^{(1)}}$ and $\bm \mu_{\mathbf f^{(1)}}$, the problem \eqref{approximate_problem1} can be decomposed into two separate optimization problems:
\begin{align}\label{RTe1}
\min_{\mathbf \Xi_{\mathbf W^{(1)}}}\frac{1}{2} \|\mathbf{\tilde B}- \mathbf \Xi_{\mathbf W^{(1)}} \mathbf{\Xi}_{\mathbf f^{(1)}}\|_{F}^2,
\end{align}
\begin{align}\label{RTe2}
\min_{\mathbf c} \frac{1}{2}\kappa(\mathbf c, \mathbf f^{(1)}),
\end{align}
which are both convex optimization problems. Taking derivative of $\kappa(\mathbf c, \mathbf f^{(1)})$ with respect to $\mathbf c$, we obtain
\begin{align}\label{RTe3}
\frac{\partial \kappa(\mathbf c, \mathbf f^{(1)})}{\partial \mathbf d}=2\mathbf M_L(\tilde{\mathbf J} \mathbf c+\tilde{\mathbf J}\mathbf W^{\operatorname{T}}\mathbf J_{1}{\mathbf W}^{(1)\operatorname{T}}\bm \mu_{\mathbf f^{(1)}}+\mathbf a_{P_Y}-\bm \mu_{\mathbf a}).
\end{align}
By setting \eqref{RTe3} to zero, we obtain \eqref{optimal_bias2}. Similarly, by setting the derivative 
\begin{align}
\frac{\partial}{\partial \mathbf \Xi_{\mathbf W^{(1)}}}\|\tilde{\mathbf B}-\mathbf \Xi_{\mathbf W^{(1)}} \mathbf \Xi_{\mathbf f^{(1)}}\|_{F}^2=2(\mathbf \Xi_{\mathbf W^{(1)}}\mathbf \Xi_{\mathbf f^{(1)}}\mathbf \Xi_{\mathbf f^{(1)}}^{\operatorname{T}}-\tilde{\mathbf B}\mathbf {\Xi}_{\mathbf f^{(1)}}^{\operatorname{T}})
\end{align}
to zero, we find \eqref{optimal_weight_2}.

{\red You usually divide a proof into many paragraphs. Paragraphs should be designed based on similarity among their functionalities.}

\section{Proof of Theorem \ref{theorem5}}\label{ptheorem5}

For fixed $\mathbf \Xi_{\mathbf W^{(1)}}$ and $\mathbf c$, the problem \eqref{approximate_problem1} can be decomposed into two separate optimization problems:
\begin{align}\label{T5e1}
\min_{\mathbf{\Xi}_{\mathbf f^{(1)}}}\frac{1}{2} \|\mathbf{\tilde B}- \mathbf \Xi_{\mathbf W^{(1)}} \mathbf{\Xi}_{\mathbf f^{(1)}}\|_{F}^2,
\end{align}
\begin{align}\label{T5e2}
\min_{\bm \mu_{\mathbf f^{(1)}}} \frac{1}{2}\kappa(\mathbf c, \mathbf f^{(1)}).
\end{align}
Setting the derivative 
\begin{align}
&\frac{\partial \kappa(\mathbf c, \mathbf f^{(1)})}{\bm \mu_{\mathbf f^{(1)}}} \nonumber\\
=&2\mathbf \Xi_{\mathbf W^{(1)}} \mathbf R_L(\mathbf a_{P_Y}-\bm \mu_{\mathbf a}+\mathbf J\mathbf c)+2 \mathbf \Xi_{\mathbf W^{(1)}}^{\operatorname{T}}\mathbf \Xi_{\mathbf W^{(1)}} \bm \mu_{\mathbf f^{(1)}}
\end{align}
to zero, we get \eqref{optimal_mean2}. Similarly, by setting the derivative
\begin{align}
&\frac{\partial}{\partial \mathbf \Xi_{\mathbf f^{(1)}}}\|\tilde{\mathbf B}-\mathbf \Xi_{\mathbf W^{(1)}} \mathbf \Xi_{\mathbf f^{(1)}}\|_{F}^2 \nonumber\\
=&2(\mathbf \Xi_{\mathbf f^{(1)}}^{\operatorname{T}}\mathbf \Xi_{\mathbf W^{(1)}}^{\operatorname{T}}\mathbf \Xi_{\mathbf W^{(1)}}-\tilde{\mathbf B}^{\operatorname{T}}\mathbf \Xi_{\mathbf W^{(1)}})
\end{align}
to zero, we get \eqref{optimal_feature2}.


\section{Proof of Theorem \ref{theorem6}}\label{ptheorem6}
Similar to Appendix \ref{ptheorem3}, one lower bound of the first term in \eqref{approximate_problem} is given by
\begin{align}\label{T6e1}
\|\mathbf{\tilde B}- \mathbf \Xi_{\mathbf W} \mathbf{\Xi}_{\mathbf f}\|_{F}^2\geq \sum_{i=k_1+1}^{K} \sigma_i^2.
\end{align}
By substituting \eqref{optimal_pair2} into \eqref{T6e1}, equality is achieved in \eqref{T6e1}. {\red Is (57) needed in the above sentence?} The optimal bias $\bar{\mathbf c}$ is in accordance with Appendix \ref{ptheorem4}.
 {\red Where is the proof of the optimal mean $\bar{\bm \mu}_{\mathbf f^{(1)}}$?}

}
\end{document}